\documentclass[11pt]{article}
\usepackage{fullpage}
\usepackage{cancel}
\usepackage{amsmath,amsfonts,amsthm,amssymb, mathrsfs}
\usepackage{dsfont}
\usepackage{enumerate}
\usepackage{cite}
\usepackage{hyperref}
\usepackage{comment}
\usepackage{float}
\usepackage{enumitem}
\usepackage[cmtip,all]{xy}
\usepackage{boxedminipage}
\usepackage[labelformat=empty]{caption}
\usepackage{enumitem}

\usepackage{breakcites}
\usepackage[table]{xcolor}

\newif\ifnotes
\notestrue

\ifnotes
\newcommand{\gal}[1]{$\ll$\textsf{\color{blue} Gal: { #1}}$\gg$}
\newcommand{\guy}[1]{$\ll$\textsf{\color{red} Guy: { #1}}$\gg$}
\else
\newcommand{\gal}[1]{}
\newcommand{\guy}[1]{}
\fi

\usepackage[textsize=tiny, textwidth=2.5cm]{todonotes}
\usepackage{marginnote}
\newcounter{todonumber}
\newcommand{\smallnote}[2][]{{%
 \let\marginpar\marginnote%
 \ifodd\value{todonumber}%
   \reversemarginpar%
 \else%
 \fi%
 \todo[#1]{#2}}%
 \stepcounter{todonumber}%
}
\newtheorem{theorem}{Theorem}[section]
\newtheorem{proposition}[theorem]{Proposition}
\newtheorem{corollary}[theorem]{Corollary}
\newtheorem{lemma}[theorem]{Lemma}
\newtheorem{claim}[theorem]{Claim}

\newtheorem{definition}[theorem]{Definition}

\newenvironment{proofof}[1]{
\setlength{\parindent}{0cm}
\noindent {\em Proof of #1.  }}
{\hfill$\Box$}

\newcommand{\A}{{\cal A}}

\newcommand{\D}{{\cal D}}
\newcommand{\F}{{\cal F}}
\newcommand{\X}{{\cal X}}
\newcommand{\Y}{{\cal Y}}

\renewcommand{\H}{{\cal H}}

\renewcommand{\L}{{\cal L}}

\newcommand{\R}{\mathbb{R}}

\newcommand{\err}{\mathit{err}}
\newcommand{\sign}{\mathit{sign}}

\let\Pr\relax
\DeclareMathOperator*{\Pr}{\textnormal{\bf Pr}}

\DeclareMathOperator*{\argmin}{argmin}

\newcommand{\poly}{{\rm poly}}
\newcommand{\zo}{\{0,1\}}
\newcommand{\lin}{\mathit{lin}}
\newcommand{\sig}{\phi}

\title{Probably Approximately Metric-Fair Learning}
\author{
Guy N. Rothblum\thanks{\href{mailto:rothblum@alum.mit.edu}{rothblum@alum.mit.edu}. Research supported by the ISRAEL SCIENCE FOUNDATION (grant No. 5219/17).}  \\ Weizmann Institute of Science
\and
Gal Yona\thanks{\href{mailto:gal.yona@gmail.com}{gal.yona@gmail.com}. Research supported by the ISRAEL SCIENCE FOUNDATION (grant No. 5219/17).} \\ Weizmann Institute of Science
}
\date{}

\begin{document}
\begin{titlepage}
\clearpage\maketitle

\thispagestyle{empty}

\begin{abstract}
The seminal work of Dwork {\em et al.} [ITCS 2012] introduced a metric-based notion of individual fairness. Given a task-specific similarity metric, their notion required that every pair of similar individuals should be treated similarly. In the context of machine learning, however, individual fairness does not generalize from a training set to the underlying population. We show that this can lead to computational intractability even for simple fair-learning tasks. 

With this motivation in mind, we introduce and study a relaxed notion of {\em approximate metric-fairness}: for a random pair of individuals sampled from the population, with all but a small probability of error, if they are similar then they should be treated similarly.
 We formalize the goal of achieving approximate metric-fairness simultaneously with best-possible accuracy as Probably Approximately Correct and Fair (PACF) Learning. We show that approximate metric-fairness {\em does} generalize, and leverage these generalization guarantees to construct polynomial-time PACF learning algorithms for the classes of linear and logistic predictors. 
\end{abstract}
\end{titlepage}

\section{Introduction}
\label{sec:intro}

Machine learning is increasingly used to make consequential classification decisions about individuals. Examples range from predicting whether a user will enjoy a particular article, to estimating a felon's recidivism risk, to determining whether a patient is a good candidate for a medical treatment. %These automated classifications influence or determine outcomes received by these individuals (suggesting a conservative Op-Ed, approving early parole, or initiating chemotherapy).
Automated classification comes with great benefits, but it also raises substantial societal concerns (cf. \cite{ONeil2016} for a recent perspective). One prominent concern is that these algorithms might discriminate against individuals or groups in a way that violates laws or social and ethical norms. This might happen due to biases in the training data or due to biases introduced by the algorithm. To address these concerns, and to truly unleash the full potential of automated classification, there is a growing need for frameworks and tools to mitigate the risks of algorithmic discrimination. A growing literature attempts to tackle these challenges by exploring different fairness criteria.

Discrimination can take many guises. It can be difficult to spot and difficult to define. Imagine a protected minority population $P$ (defined by race, gender identity, political affiliation, etc). A natural approach for protecting the members of $P$ from discrimination is to make sure that they are not mistreated {\em on average}. For example, that on average members of $P$ and individuals outside of $P$ are classified in any particular way with roughly the same probability. This is a {\em ``group-level''} notion of fairness, sometimes referred to as {\em statistical parity}.

Pointing out several weakness of group-level notions of fairness, the seminal work of \cite{DworkHPRZ12} introduced a notion of {\em individual fairness}. Their notion relies on a {\em task-specific similarity metric} that specifies, for every two individuals, how similar they are with respect to the specific classification task at hand. Given such a metric, similar individuals should be treated similarly, i.e. assigned similar classification distributions (their focus was on probabilistic classifiers, as will be ours). In this work, we refer to their fairness notion as {\em perfect metric-fairness}.
%difference in the probability of any specific classification task for every two individuals is bounded by their distance (as measured by the metric).

Given a good metric, perfect metric-fairness provides powerful protections from discrimination. Furthermore, the metric provides a vehicle for specifying social norms, cultural awareness, and task-specific knowledge. While coming up with a good metric can be challenging, metrics arise naturally in prominent existing examples (such as credit scores and insurance risk scores), and in natural scenarios (a metric specified by an external regulator). Dwork {\em et al.} studied the goal of finding a (probabilistic) classifier that minimizes utility loss (or maximizes accuracy), subject to satisfying the perfect metric-fairness constraint. They showed how to phrase and solve this optimization problem for a given collection of individuals.

\subsection{This Work: Approximately Metric-Fair Machine Learning}
\label{sec:intro:this_work}

Building on these foundations, we study {\em metric-fair machine learning}. Consider a learner that is given a similarity metric and a training set of labeled examples, drawn from an underlying population distribution. The learner should output a {\em fair} classifier that (to the extent possible)  accurately classifies the underlying population.

This goal departs from the scenario studied in \cite{DworkHPRZ12}, where the focus was on guaranteeing metric-fairness and utility for the dataset at hand. {\em Generalization} of the fairness guarantee is a key difference: we focus on guaranteeing fairness not just for the (training) data set at hand, but also for the underlying population from which it was drawn. We note that perfect metric-fairness does not, as a rule, generalize from a training set to the underlying population. This presents computational difficulties for constructing learning algorithms that are perfectly metric-fair for the underlying population. Indeed, we exhibit a simple learning task that, while easy to learn without fairness constraints, becomes computationally infeasible under the perfect metric-fairness constraint (given a particular metric).\footnote{We remark that perfect metric-fairness can always be obtained trivially by outputting a constant classifier that treats all individuals identically, the challenge is achieving metric-fairness together with non-trivial accuracy.} See below and in Section \ref{sec:intro:hardness} for further details.

We develop a relaxed {\em approximate metric-fairness} framework for machine learning, where fairness does generalize from the training set to the underlying population, and present polynomial-time fair learning algorithms in this framework. We proceed to describe our setting and contributions.

\paragraph{Problem setting.}  A metric-fair learning problem is defined by a domain $\X$ and a similarity metric $d$.
%and a class $\H$ of probabilistic classifiers $h:\X\rightarrow [0,1]$.
A metric-fair learning algorithm gets as input the metric $d$ and a sample of labeled examples, drawn i.i.d. from a distribution $\D$ over labeled examples from $(\X \times \pm 1)$, and outputs a classifier $h$. To accommodate fairness, we focus on probabilistic classifiers $h: \X \rightarrow [0,1]$, where we interpret $h(x)$ as the probability of label 1 (the probability of $-1$ is thus $(1-h(x))$). We refer to these probabilistic classifiers as \emph{predictors}.

\paragraph{Approximate Metric-Fairness.} Taking inspiration from Valiant's celebrated PAC learning model \cite{Valiant84}, we allow a small fairness error, which opens the door to generalization. We require that for two individuals sampled from the underlying population, with all but a small probability, if they are similar then they should be treated similarly. Similarity is measured by the statistical distance between the classification distributions given to the two individuals (we also allow a small additive slack in the similarity measure). We refer to this condition as {\em approximate metric-fairness (MF)}. Similarly to PAC learning, we also allow a small probability of a complete fairness failure.

Given a well-designed metric, approximate metric-fairness guarantees that almost every individual gets fair treatment compared to almost every other individual. In particular, it provides discrimination-protections to {\em every} group $P$ that is not too small. However, this guarantee also has limitations: particular individuals and even small groups might encounter bias and discrimination. There are certainly settings in which this is problematic, but in other settings protecting all groups that are not too small is an appealing guarantee. The relaxation is well-motivated because approximate fairness opens the door to fairness-generalization bounds, as well as efficient learning algorithms for a rich collection of problems (see below). We elaborate on these choices and their consequences in Section \ref{sec:intro:approx_fair}.

\paragraph{Competitive accuracy.} Turning our attention to the accuracy objective, we follow \cite{DworkHPRZ12} in considering fairness to be a hard constraint (e.g. imposed by a regulator). Given the fairness constraint, what is a reasonable accuracy objective? Ideally, we would like the predictor's accuracy to approach (as the sample size grows) that of the most accurate approximately MF predictor. This is analogous to the accuracy guarantee pioneered in \cite{DworkHPRZ12}. A {\em probably approximately correct and fair (PACF)} learning algorithm guarantees both approximate MF and ``best-possible'' accuracy. A more relaxed accuracy benchmark is approaching the accuracy of the best classifier that is approximately MF for a tighter (more restrictive) fairness-error. We refer this as a {\em relaxed} PACF learning algorithm  (looking ahead, our efficient algorithms achieve this relaxed accuracy guarantee). We note that even relaxed PACF guarantees that the classifier is (at the very least) competitive with the best {\em perfectly} metric-fair classifier. We elaborate in Section \ref{sec:intro:PACF}.

\paragraph{Generalization bounds.} A key issue in learning theory is that of generalization: to what extent is a classifier that is accurate on a finite sample $S \sim \D^m$ also guaranteed to be accurate w.r.t the underlying distribution? We develop strong generalization bounds for approximate metric-fairness, showing that for any class of predictors with bounded Rademacher complexity, approximate MF on the sample $S$ implies approximate MF on the underlying distribution (w.h.p. over the choice of sample $S$). The use of Rademacher complexity guarantees fairness-generalization for finite classes and also for many infinite classes. Proving that approximate metric-fairness generalizes well is a crucial component in our analysis: it opens the door to polynomial-time algorithms that can focus on guaranteeing fairness (and accuracy) on the sample. Generalization also implies information-theoretic sample-complexity bounds for PACF learning that are similar to those known for PAC learning (without any fairness constraints). We elaborate in Section \ref{sec:intro:generalization}.

\paragraph{Efficient algorithms.} We construct polynomial-time (relaxed) PACF algorithms for linear and logistic regression. Recall that (for fairness) we focus on regression problems: learning predictors that assign a probability in $[0,1]$ to each example. For linear predictors, the probability is a linear function of an example's distance from a hyperplane. Logistic predictors compose a linear function with a sigmoidal transfer function. This allows logistic predictors to exhibit sharper transitions from low predictions to high predictions. In particular, a logistic predictor can better approximate a classifier that labels examples that are below a hyperplane by $-1$, and examples that are above the hyperplane by 1.
Linear and logistic predictors can be more powerful than they first seem: by embedding a learning problem into a higher-dimensional space, linear functions (over the expanded space) can capture the power of many of the function classes that are known to be PAC learnable \cite{HellersteinS07}. We overview these results in Section \ref{sec:intro:algorithms}. We note that a key challenge in efficient metric-fair learning is that the fairness constraints are neither Lipschitz nor convex (even when the predictor is linear). This is also a challenge for proving generalization and sample complexity bounds. Berk {\em et al.} \cite{berk2017convex} also study fair regression and formulate a measure of individual fairness loss, albeit in a different setting without a metric (see Section \ref{sec:intro:related_work}).

%There, we overcome this difficulty by employing a piecewise-linear and Lipschitz approximation to a measure of {\em metric-fairness loss}. This approximation, however, is not convex. To efficiently learn linear predictors, we use a convex and Lipschitz surrogate for the fairness loss (the use of this surrogate results in our relaxed utility guarantees). Learning logistic predictors is considerably more challenging: not only is the fairness loss not convex, but the sigmoidal transfer function is also not convex. We overcome this obstacle by embedding the learning problem into a higher-dimensional space, where logistic predictors and their fairness constraints can be {\em approximated} by convex expressions. The embedding is derived from a result of Shalev-Schwartz {\em et al.} \cite{shalev2011learning}. The resulting learning algorithm is polynomial so long as the Lipschitz constant $L$ of the sigmoidal transfer function is constant (more generally, it is exponential in $L$). Section \ref{sec:intro:algorithms} gives a more detailed overview of our contributions and techniques.

\paragraph{Perfect metric-fairness is hard.} Under mild cryptographic assumptions, we exhibit a learning problem and a similarity metric where: $(i)$ there exists a {\em perfectly fair and perfectly accurate} simple (linear) predictor, but $(ii)$ any polynomial-time perfectly metric-fair learner can only find a trivial predictor, whose error approaches 1/2.  In contrast, $(iii)$ there {\em does} exist a polynomial-time (relaxed) PACF learning algorithm for this task.
This is an important motivation for our study of {\em approximate} metric-fairness. We elaborate in Section \ref{sec:intro:hardness}.

\paragraph{Organization.} In the remainder of this section we provide an overview of our contributions. {\bf Section \ref{sec:intro:approx_fair}} details and discusses the definition of approximate metric-fairness and its relationship to related works. Accurate and fair (PACF) learning is discussed in {\bf Section \ref{sec:intro:PACF}}. We state and prove fairness-generalization bounds in {\bf Section \ref{sec:intro:generalization}}. Our polynomial-time PACF learning algorithms for linear and logistic regression are in {\bf Section \ref{sec:intro:algorithms}}. {\bf Section \ref{sec:intro:hardness}} elaborates on the hardness of {\em perfectly} metric-fair learning. Further related work is discussed in {\bf Section \ref{sec:intro:related_work}}.

Full and formal details are in {\bf Sections \ref{sec:fairness_defs} through \ref{sec:hardness}}. Conclusions and a discussion of future directions are in Section {\bf \ref{sec:conclusions}}.

\subsection{Approximate Metric-Fairness}
\label{sec:intro:approx_fair}

We require that metric-fairness holds for all but a small $\alpha$ fraction of pairs of individuals. That is, with all but $\alpha$ probability over a choice of two individuals from the underlying distribution, if the two individuals are similar then they get similar classification distributions. We think of $\alpha \in [0,1)$ as a small constant, and note that setting $\alpha=0$ recovers the definition of {\em perfect} metric-fairness (thus, setting $\alpha$ to be a small constant larger than 0 is indeed a relaxation). Similarity is measured by the statistical distance between the classification distributions given to the two individuals, where we also allow a small additive slack $\gamma$ in the similarity measure. The larger $\gamma$ is, the more ``differently'' similar individuals might be treated. We think of $\gamma$ as a small constant, close to 0.

\begin{definition}
A predictor $h$ is $(\alpha,\gamma$) approximately metric-fair (MF) with respect to a similarity metric $d$ and a data distribution $\D$ if:
\begin{align}
\label{eq:fairness_loss}
\L^{F}_{\gamma} \triangleq \Pr_{x,x' \sim \D}[| h(x) - h(x') | > d(x,x') + \gamma ] \leq \alpha
\end{align}
\end{definition}

%The {\em MF loss} of the predictor $h$ on the pair $(x,x')$ is 1 if the (internal) inequality in Equation \eqref{eq:fairness_loss} holds, and 0 otherwise (hence, we refer to this as a $0/1$ fairness loss). A {\em predictor} is $(\alpha,\gamma)$-approximately MF if with all but $\alpha$ probability over two individuals $(x,x')$ sampled from $\D$, its respective MF loss is 0.

Similarly to the PAC learning model, we also allow a small $\delta$ probability of failure. This probability is taken over the choice of the training set and over the learner's coins. For example, $\delta$ bounds the probability that the randomly sampled training set is not representative of the underlying population. We think of $\delta$ as very small or even negligible. A learning algorithm is {\em probably approximately metric-fair} if with all but $\delta$ probability over the sample (and the learner's coins), it outputs a classifier that is $(\alpha,\gamma)$-approximately MF. Further details are in Section \ref{sec:fairness_defs}.

Given a well-designed metric, approximate metric-fairness (for sufficiently small $\alpha,\gamma$) guarantees that almost every individual gets fair treatment compared to almost every other individual (see Section \ref{sec:AMF_interpretation} for a quantitative discussion). {\em Every} protected group $P$ of fractional size significantly larger than $\alpha$ is protected in the sense that, on average, members of $P$ are treated similarly to similar individuals outside of $P$. We note, however, that this guarantee does not protect single individuals or small groups (see the discussion in Section \ref{sec:intro:this_work}).

\paragraph{Between group and individual fairness: related works.} Recent works \cite{hebert2017calibration,kearns2017preventing} study fairness notions that aim to protect large collections of sufficiently-large groups. Similarly to our work, these can be viewed as falling between individual and group notions of fairness. A distinction from these works is that approximate metric-fairness protects {\em every} sufficiently-large group, rather than a large collection of groups that is fixed a priori. Recent works \cite{GillenJKR18,KimRR18} extend the study of metric fairness to settings where the metric is not known (whereas we focus on a setting where the metric is fixed and known in its entirety), and consider relaxed fairness notions that allow individual fairness to be violated.

\subsection{Accurate and Fair Learning}
\label{sec:intro:PACF}

Our goal is to obtain learning algorithms that are probably approximately metric-fair, and that simultaneously guarantee non-trivial accuracy. Recall that fairness, on its own, can always be obtained by outputting a constant predictor that ignores its input and treats all individuals identically (indeed, such a classifier is {\em perfectly} metric-fair). It is the combination of the fairness and the accuracy objectives that makes for an interesting task. As discussed above, we follow \cite{DworkHPRZ12} in focusing on finding a predictor that maximizes accuracy, subject to the approximate metric-fairness constraint. This is a natural formulation, as we think of fairness as a hard requirement (imposed, for example, by a regulator), and thus fairness cannot be traded off for better accuracy.

As discussed above, we focus on the setting of binary classification. A {\em learning problem} is defined by an instance domain $\X$ and a class $\H$ of predictors (probabilistic classifiers) $h:\X\rightarrow [0,1]$. A {\em fair} learning problem also includes a similarity metric $d: \X^2 \rightarrow [0,1]$. The learning algorithm gets as input the metric $d$ and a sample of labeled examples, drawn i.i.d. from a distribution $\D$ over labeled examples from $(\X \times \pm 1)$, and its goal is to output a predictor that is both fair and as accurate as possible. %To accommodate the fairness constraints, we allow the learned classifier $h$ to return real values in $[0,1]$, where we interpret $h(x)$ as the probability that the label is 1. We refer to such probabilistic classifiers as {\em predictors}. 
A {\em proper} learner outputs a predictor in the class $\H$, whereas an {\em improper} learner's output is unconstrained  (but $\H$ is used as a benchmark for accuracy). For a learned (real-valued) predictor $h$, we use $\mathit{err}_D(h)$ to denote the expected $\ell_1$ error of $h$ (the absolute loss) on a random sample from $\D$.\footnote{All results also translate to $\ell_2$ error (the squared loss).}
%I.e. the expected absolute value of the difference between the (real-valued) classification assigned by $h$ and the true label of the sample.

\paragraph{Accuracy guarantee: PACF learning.}As discussed above, the goal in metric-fair and accurate learning is optimizing the predictor's accuracy subject to the fairness constraint. Ideally, we aim to approach (as the sample size grows) the error rate of the most accurate classifier that satisfies the fairness constraints. A more relaxed benchmark is guaranteeing $(\alpha,\gamma)$-approximate metric-fairness, while approaching the accuracy of the best classifier that is $(\alpha',\gamma')$-approximately metric-fair, for $\alpha' \in [0,\alpha]$ and $\gamma' \in [0,\gamma]$. Our efficient learning algorithms will achieve this more relaxed accuracy goal (see below). We note that even relaxed competitiveness means that the classifier is (at the very least) competitive with the best {\em perfectly} metric-fair classifier.

These goals are captured in the following definition of {\em probably approximately correct and fair (PACF) learning}. Crucially, both fairness and accuracy goals are stated with respect to the (unknown) underlying distribution.

\begin{definition}[PACF Learning] A learning algorithm $\A$ PACF-learns a hypothesis class $\mathcal{H}$ if for every metric $d$ and population distribution $\D$, every required fairness parameters $\alpha, \gamma \in [0,1)$, every failure probability $\delta\in(0,1)$, and every error parameters $\epsilon, \epsilon_\alpha, \epsilon_\gamma\in(0,1)$ the following holds:

There exists a sample complexity $m=\poly\left({\frac{\log|\X| \cdot \log(1/\delta)}{\alpha \cdot \gamma \cdot \epsilon_ \cdot \epsilon_{\alpha} \cdot \epsilon_{\gamma}}}\right)$ and constants $\alpha',\gamma' \in [0,1)$ (specified below), such that with all but $\delta$ probability over an i.i.d. sample  of size $m$ and $\A$'s coin tosses, the output predictor $h$ satisfies the following two conditions:
\begin{enumerate}
\item {\bf Fairness}: $h$ is $(\alpha,\gamma)$-approximately metric-fair w.r.t. the metric $d$ and the distribution $\D$.

\item {\bf Accuracy}: Let $\mathcal{H}_F'$ denote the subclass of hypotheses in $\mathcal{H}$ that are $(\alpha'-\epsilon_{\alpha},\gamma'-\epsilon_{\gamma})$-approximately metric-fair, then: $$\mathit{err}_D(h) \leq \min_{h' \in \mathcal{H}_F'} \mathit{err}_D(h') + \epsilon$$
\end{enumerate}

We say that $\A$ is {\em efficient} if it runs in time $\poly(m)$. If accuracy holds for $\alpha'=\alpha$ and $\gamma'=\gamma$, then we stay that $\A$ is a {\em strong} PACF learning algorithm. Otherwise, we say that $\A$ is a {\em relaxed} PACF learning algorithm.
%If there exists a function $g:[0,1)^2 \rightarrow [0,1)^2$ such that accuracy holds for $(\alpha',\gamma') \succeq g(\alpha,\gamma)$,\footnote{The ``$\succeq$'' notation means that $\alpha'$ is no greater than than the first component of $g(\alpha,\gamma)$ and $\gamma'$ is no greater than the second component.} then we say that $\A$ is a $g$-relaxed PACF learning algorithm.
\end{definition}

See Section \ref{sec:accuracy} and Definitions \ref{def:pacf} and \ref{def:weak_pacf} for a full treatment. Note that the accuracy guarantee is {\em agnostic}: we make no assumptions about the way the training labels are generated. Agnostic learning is particularly well suited to our metric-fairness setting: since we make no assumptions about the metric $d$, even if the labels are generated by $h \in \H$, it might be the case that $d$ does not allow for accurate predictions, in which case a fair learner cannot compete with $h$'s accuracy.

\subsection{Generalization}
\label{sec:intro:generalization}

Generalization is a key issue in learning theory. We develop strong generalization bounds for approximate metric-fairness, showing that with high probability, guaranteeing {\em empirical} approximate MF on a training set also guarantees approximate MF on the underlying distribution (w.h.p. over the choice of sample $S$). This generalization bound opens the door to polynomial-time algorithms that can focus on guaranteeing fairness (and accuracy) on the sample and effectively rules out the possibility of creating a ``false facade'' of fairness (i.e, a classifier that appears fair on a random sample, but is not fair w.r.t new individuals).

Towards proving generalization, we define the empirical fairness loss on a sample $S$ (a training set). Fixing a fairness parameter $\gamma$, a predictor $h$ and a pair of individuals $x,x'$ in the training set, consider the MF loss on the ``edge'' between $x$ and $x'$ (recall that the MF loss is 1 if the ``internal'' inequality of Equation \eqref{eq:fairness_loss} holds, and 0 otherwise). Observe that the losses on the $\binom{|S|}{2}$ edges are not independent random variables (over the choice of $S$), because each individual $x \in S$ affects many edges. Thus, rather than count the empirical MF loss over all edges, we restrict ourselves to a ``matching'' $M(S)$ in the complete graph whose vertices are $S$: a collection of edges, where each individual is involved in exactly one edge. The empirical MF loss of $h$ on $S$ is defined as the average MF loss over edges in $M(S)$.\footnote{The choice of {\em which} matching is used does not affect any of the results. Note that we could also choose to average over {\em all} the edges in the graph induced by $S$. Generalization bounds still follow, but the rate of convergence is not faster than restricting our attention to a matching.} Note that, since we restricted our attention to a matching, the MF losses on these edges are now independent random variables (over the choice of $S$). A classifier is {\em empirically} $(\alpha,\gamma)$-approximately MF if its empirical MF loss is at most $\alpha$. We are now ready to state our generalization bound:

\begin{theorem}
\label{thm:intro_generalization}
Let $\mathcal{H}$ be a hypothesis class with Rademacher complexity $R_m(\mathcal{H}) = (r/\sqrt{m})$. For every $\delta \in(0,1)$ and every $\epsilon_{\alpha},\epsilon_{\gamma} \in (0,1)$,  there exists a sample complexity $m = O\left( \frac{r^2 \cdot \ln(1/\delta)}{\epsilon^2_{\alpha} \cdot \epsilon^2_{\gamma}}  \right)$, such that the following holds:

%\begin{center}
With probability at least $1-\delta$ over an i.i.d sample $S\sim\D^m$, simultaneously for every $h\in\mathcal{H}$: if $h$ is $(\alpha,\gamma)$-approximately metric-fair on the sample $S$, then $h$ is also $(\alpha + \epsilon_{\alpha}, \gamma + \epsilon_{\gamma})$-approximately metric-fair on the underlying distribution $\D$.
%\end{center}
\end{theorem}

See Section \ref{sec:fairness_generalization} and Theorem \ref{thm:rad_generalization} for a full statement and discussion (and see Definition \ref{def:Rademacher} for a definition of Rademacher complexity). Rademacher complexity differs from the celebrated VC-dimension in several respects: first, it is defined for any class of real-valued functions (making it suitable for our setting of learning probabilistic classifiers);  second, it is data-dependent and can be measured from finite samples (indeed, Theorem \ref{thm:intro_generalization} can be stated w.r.t. the {\em empirical} Rademacher complexity on a given sample); third, it often results in tighter uniform convergence bounds (see, e.g, ~\cite{koltchinskii2002empirical}). We note that for every finite hypothesis class $\mathcal{H}$ whose range is $[0,1]$, the Rademacher complexity is bounded by $O(\sqrt{\log|\mathcal{H}|/m})$.

\paragraph{Technical Overview of Theorem \ref{thm:intro_generalization}.} For any class of (bounded) real-valued functions $\F$, the maximal difference (over all functions $f\in \F$) between the function's empirical average on a randomly drawn sample, and the function's true expectation over the underlying distribution, can be bounded in terms of the Rademacher complexity of the class (as well as the sample size and desired confidence). For a hypothesis class $\mathcal{H}$ and a loss function $\ell$, applying this result for the class $\mathcal{L}(\mathcal{H})=\left\{ \ell_{h}\right\} _{h\in\mathcal{H}}$ yields a bound on the maximal difference (over all hypotheses $h\in \H$) between the true loss and the empirical loss, in terms of the Rademacher complexity of the composed class $\mathcal{L}(\mathcal{H})$. If the loss function $\ell$ is $G$-Lipschitz, this can be converted to a bound in terms of the Rademacher complexity of $\mathcal{H}$ using the fact that $R\left(\mathcal{L}(\mathcal{H})\right)\leq G\cdot R\left(\mathcal{H}\right)$.

Turning our attention to generalization of the fairness guarantee, we are faced with the problem that our ``0-1'' MF loss function is \emph{not Lipschitz}. We resolve this by defining an approximation $\ell'$ to the MF loss that is a piece-wise linear and $G$-Lipschitz function. The approximation $\ell'$ {\em does} generalize, and so we conclude that the empirical MF loss is close to the empirical value of $\ell'$, which is close to the {\em true} value of $\ell'$, which in turn is close to the {\em true} MF loss. The approximation incurs a $1/G$ additive slack in the fairness guarantee. The larger $G$ is, the more accurately $\ell'$ approximates the MF loss, but this comes at the price of increasing the Lipschitz constant (which hurts generalization). The generalization theorem statement above reflects a choice of $G$ that trades off these conflicting concerns.

%\guy{review: } We note that this is an arbitrarily good approximation of our MF loss, in the sense that the 0-1 fairness loss cannot exceed the fairness loss that is calculated on the Lipschitz proxy, with an added term of $\frac{1}{G}$ in the ``slack'' parameter $\gamma$. We use the above Rademacher-based generalization argument to prove that our Lipschitz loss indeed generalize well; this implies that approximate MF also generalizes, up to an extra $\frac{1}{G}$ additive factor in $\gamma$.

\subsubsection{Information-Theoretic Sample Complexity}
\label{sec:intro:information-theoretic}

The fairness-generalization result of Theorem \ref{thm:intro_generalization} implies that, from a sample-complexity perspective, any hypothesis class is strongly PACF learnable, with sample complexity comparable to that of standard PAC learning. An exponential-time PACF learning algorithm simply finds the predictor in $\H$ that minimizes the empirical error, while also satisfying empirical approximate metric-fairness.

\begin{theorem}
Let $\mathcal{H}$ be a hypothesis class with Rademacher complexity $R_m(\mathcal{H}) = (r/\sqrt{m})$. Then $\mathcal{H}$ is information-theoretically strongly PACF learnable with sample complexity $m=O\left(\frac{r^{2}\ln (1/\delta)}{\left(\epsilon'\right)^{2}}\right)$, for $\epsilon'=\min\left\{ \epsilon,\epsilon_{\alpha},\epsilon_{\gamma}\right\}$ .
\end{theorem}

\subsection{Efficient Fair Learning}
\label{sec:intro:algorithms}

One of our primary contributions is the construction of polynomial-time relaxed-PACF learning algorithms for expressive hypothesis classes. We focus on linear classification tasks, where the labels are determined by a separating hyperplane. Learning linear classifiers, also referred to as halfspaces or linear threshold functions, is a central tool in machine learning. By embedding a learning problem into a higher-dimensional space, linear classifiers (over the expanded space) can capture surprisingly strong classes, such as polynomial threshold functions (see, for example, the discussion in \cite{HellersteinS07}).  The ``kernel trick'' (see, e.g, \cite{shalev2014understanding}) can allow for efficient solutions even over very high (or infinite) dimensional embeddings. Many of the known (distribution-free) PAC learning algorithms can be derived by learning linear threshold functions \cite{HellersteinS07}.

Recall that in {\em metric-fair} learning, we aim to learn a probabilistic classifier, or a {\em predictor}, that outputs a real value in $[0,1]$. We interpret the output as the probability of assigning the label $1$. We are thus in the setting of {\em regression}. We show polynomial-time relaxed-PACF learning algorithms for {\em linear regression} and for {\em logistic regression}. See Section \ref{sec:learning-linear} for full and formal details.

\subsubsection{Linear Regression}
\label{sec:intro:linear-regression}

Linear regression, the task of learning linear predictors, is an important and well-studied problem in the machine learning literature. In terms of accuracy, this is an appealing class when we expect a linear relationship between the probability of the label being $1$ and the distance from a hyperplane. Taking the domain $\X$ to be the unit ball, we define the class of linear predictors as:
\begin{align*}
H_{\lin} \overset {{\scriptscriptstyle {\scriptstyle \text{def}}}}{=}\left\{ \mathbf{x}\mapsto \frac{1 + \left\langle \mathbf{w},\mathbf{x}\right\rangle}{2} :\,\,\|\mathbf{w}\|\leq1\right\},
\end{align*}
We restrict $w$ to the unit ball to guarantee that  $\left\langle \mathbf{w},\mathbf{x}\right\rangle \in [-1,1]$. We then invoke a linear transformation so that the final prediction is in $[0,1]$, as required. Restricting the predictor's output to  the range $[0,1]$ is important. In particular, it means that a linear predictor must be $(1/2)$-Lipschitz, which might not be appropriate for certain classification tasks (see the discussion of logistic regression below).

We show a relaxed PACF learning algorithm for $H_{\lin}$:

\begin{theorem}
\label{thm:intro_linear}
$H_{\lin}$ is relaxed PACF learnable with sample and time complexities of $poly(\frac{1}{\epsilon_\gamma}, \frac{1}{\epsilon_\alpha}, \frac{1}{\epsilon}, \log \frac{1}{\delta})$. For every $\gamma' \in [0,1)$ and $\alpha' = (\alpha \cdot \gamma - \gamma')$, the accuracy of the learned predictor approaches (or beats) the most accurate  $(\alpha',\gamma')$-approximately MF predictor.
\end{theorem}

\paragraph{Algorithm overview.} Since the Rademacher complexity of (bounded) linear functions is small \cite{kakade2009complexity}, Theorem  \ref{thm:intro_generalization} implies that empirical approximate metric-fairness on the training set generalizes to the underlying population. Thus, given the metric and a training set, our task is to find a linear predictor that is as accurate as possible, conditioned on the {\em empirical} fairness constraint. We use $H = H_{\lin}$ to denote the class of linear predictors defined above. Fixing desired fairness parameters $\alpha,\gamma \in (0,1)$, let $\widehat{H}^{\alpha,\gamma} \subseteq H$ be the subset of linear functions that are also $(\alpha,\gamma)$-approximately MF on the training set. Given a training set $S$ of $m$ labeled examples, we would like to solve the following optimization problem:
\begin{eqnarray*}
\argmin_{h \in H} \err_S(h) \,\, \text{subject to} \,\, h \in \widehat{H}^{\alpha,\gamma}
\end{eqnarray*}
Observe, however, that $\widehat{H}^{\alpha,\gamma}$ is not a convex set. This is a consequence of the ``$0/1$'' metric-fairness loss. Thus, we do not know how to solve the above optimization problem efficiently. Instead, we will further constrain the predictor $h$ by bounding its {\em $\ell_1$ MF loss}. For a predictor $h$ let its (empirical) $\ell_1$ MF violation $\xi_S(h)$ be given by:
\begin{eqnarray*}
\xi_S(h) = \sum_{(x,x') \in M(S)} \max \left( 0, |h(x) - h(x')| - d(x,x') \right).
\end{eqnarray*}
For $\tau \in [0,1]$, we  take $\widehat{H}^{\tau}_{\ell_1} \subset H$ to be the set of linear predictors $h$ s.t. $\xi_S(h) \leq \tau$. For any fixed $\tau$, this is a convex set, and we can find the most (empirically) accurate predictor in $\widehat{H}^{\tau}_{\ell_1}$ in polynomial time. For fairness, we show that small $\ell_1$ fairness loss also implies the standard notion of approximate metric-fairness (with related parameters $\alpha,\gamma$). For accuracy, we also show that approximate metric-fairness (with smaller fairness parameters) implies small $\ell_1$ loss. Thus, optimizing over predictors whose $\ell_1$ loss is bounded gives a predictor that is competitive with (a certain class of) approximately MF predictors. In particular for $\tau,\sigma \in [0,1)$ we have:
\begin{eqnarray*}
\widehat{H}^{\tau-\sigma,\sigma} \subseteq \widehat{H}_{\ell_{1}}^{\tau} \subseteq \widehat{H}^{\frac{\tau}{\gamma},\gamma}
\end{eqnarray*}
Thus, by picking $\tau = \alpha \cdot \gamma$ we guarantee (empirical) $(\alpha,\gamma)$-approximate metric-fairness. Moreover, for any choice of $\sigma$, the set over which we optimize contains all of the predictors that are  $((\alpha \gamma - \sigma),\sigma)$-approximately MF. Thus, our (empirical) accuracy is competitive with all such predictors, and we obtain a relaxed PACF algorithm. The empirical fairness and accuracy guarantees generalize beyond the training set by Theorem \ref{thm:intro_generalization} (fairness-generalization) and a standard uniform convergence argument for accuracy.

\subsubsection{Logistic Regression}
\label{sec:intro:logistic-regression}

Logistic regression is another appealing class.  Here, the prediction need not a be a linear function of the distance from a hyperplane. Rather, we allow the use of a sigmoid function $\phi_{\ell}:\,\left[-1,1\right]\rightarrow\left[0,1\right]$ defined as $\phi_{\ell}(z)=\frac{1}{1+\exp(-4\ell \cdot z)}$ (which is continuous and $\ell$-Lipschitz). The class of logistic predictors is formed by composing a linear function with a sigmoidal transfer function:
\begin{equation}
	H_{\phi,L}\overset{{\scriptscriptstyle {\scriptstyle \text{def}}}}{=}\left\{ \mathbf{x}\mapsto\phi_{\ell}\left(\left\langle w,\mathbf{x}\right\rangle \right):\,\|\mathbf{w}\| \leq 1,\, \ell \in [0,L]\right\}
\end{equation}
The sigmoidal transfer function gives the predictor the power to exhibit sharper transitions from low predictions to high predictions around a certain distance (or decision) threshold. For example, suppose a distance from the hyperplane provides a quality score for candidates with respect to a certain task. Suppose also that an employer wants to hire candidates whose quality scores are above some threshold $\eta \in [-1,+1]$. The class $H_{\sig,L}$ can give probabilities close to 0 to candidates whose quality scores are under $\eta - 1/L$, and probabilities close to 1 to candidates whose quality scores are over $\eta + 1/L$. Linear predictors, on the other hand, need to be $(1/2)$-Lipschitz (since we restrict their output to be in $[0,1]$, see Section \ref{sec:intro:linear-regression}).\footnote{This might, at first glance, seem like a technicality. After all, why not simply consider linear predictors whose output can be in a larger range? The problem is that it isn't clear how to plug these larger values into the fairness constraints in a way that keeps the optimization problem convex and also has competitive accuracy.} Logistic predictors seem considerably better-suited to this type of scenario. Indeed, the class $H_{\phi,L}$ can achieve good accuracy on linearly separable data whose margin (i.e. the expected distance from the hyperplane) is larger than $1/L$. Moreover, similarly to linear threshold functions, logistic regression can be applied after embedding the learning problem into a higher-dimensional space. For example, in the ``quality score'' example above, the score could be computed by a low-degree polynomial.

%Thus, logistic regression seems powerful and appealing for the fair learning setting.
Our primary technical contribution is a polynomial-time relaxed PACF learner for $H_{\sig,L}$ where $L$ is constant.

\begin{theorem}
For every constant $L > 0$, $H_{\sig,L}$ is relaxed PACF learnable with sample and time complexities of $poly(\frac{1}{\epsilon_\gamma}, \frac{1}{\epsilon_\alpha}, \frac{1}{\epsilon}, \log \frac{1}{\delta})$. For every $\gamma' \in [0,1)$ and $\alpha' = (\alpha \cdot \gamma - \gamma')$, the accuracy of the learned predictor approaches (or beats) the most accurate  $(\alpha',\gamma')$-approximately MF predictor.
\end{theorem}
More generally, our algorithm is exponential in the parameter $L$. Recall that we expect to have good accuracy on linearly separable data whose margins are larger than $(1/L$). Thus, one can interpret the algorithm as having runtime that is exponential in the reciprocal of the (expected) margin.

\paragraph{Algorithm overview.}
We note that fair learning of logistic predictors is considerably more challenging than the linear case, because the sigmoidal transfer function specifies non-convex fairness constraints. In standard logistic regression, where fairness is not a concern, polynomial-time learning is achieved by replacing the standard loss with a convex logistic loss. In metric-fair learning, however, it not clear how to replace the sigmoidal transfer function by a convex surrogate.

To overcome these barriers, we use {\em improper} learning. We embed the linear problem at hand into a higher-dimensional space, where logistic predictors and their fairness constraints can be approximated by convex expressions. To do so, we use a beautiful result of Shalev-Schwartz {\em et al.} \cite{shalev2011learning} that presents a particular infinite-dimensional kernel space where our fairness constraints can be made convex.

In particular, we replace the problem of PACF learning $H_{\phi,L}$ with the problem of PACF learning $H_{B}$, a class of linear predictors with norm bounded by B in a RHKS defined by Vovk's infinite-dimension polynomial kernel, $k(x,x')=\left(1-\left\langle x,x'\right\rangle \right)^{-1}$. We learn the linear predictor in this RHKS using the result of Theorem \ref{thm:intro_linear} to obtain a relaxed PACF algorithm for $H_{B}$. We use the kernel trick to argue that the sample complexity is $m = O(B / (\epsilon')^2)$, where $\epsilon' = \min(\epsilon,\epsilon_{\alpha},\epsilon_{\gamma})$, and the time complexity is $poly(m)$.

For every $B\geq0$, we can thus learn a linear predictor (in the above RHKS) that is (empirically) sufficiently fair, and whose (empirical) accuracy is competitive with all the linear predictors with norm bounded by $B$ that are $\left(\left(\alpha\gamma-\sigma\right),\sigma\right)$-approximately MF, for any choice of $\sigma$. To prove PACF learnability of $H_{\phi,L}$, we build on the polynomial approximation result of Shalev-Schwartz {\em et al.} \cite{shalev2011learning} to show that taking $B$ to be sufficiently large
%(in particular, exponential in $L$, the Lipschitz constant of $\phi$)
ensures that the accuracy of the set of $\left(\alpha,\gamma\right)$-AMF predictors in $H_{\phi,L}$ is comparable to the accuracy of the set of $\left(\alpha,\gamma\right)$-AMF predictors in $H_{B}$. This requires a choice of $B$ that is $\exp(O(L \cdot\ln (L/\epsilon'))$, which is where the exponential dependence on $L$ comes in.

\subsection{Hardness of Perfect Metric-Fairness}
\label{sec:intro:hardness}

As discussed above, perfect metric-fairness {\em does not generalize} from a training set to the underlying population. For example, consider a very small subset of the population that isn't represented in the training set. A classifier that discriminates against this small subset might be perfectly metric-fair {\em on the training set}. The failure of generalization poses serious challenges to constructing learning algorithms.
Indeed, we show that perfect metric-fairness can make simple learning tasks computationally intractable (with respect to a particular metric).
%We present a natural learning problem and a metric, where a learner can tell that a particular (linear) classifier is {\em empirically} perfectly fair (and perfectly accurate). However, even though the classifier is perfectly fair on the underlying distribution, the (polynomial-time) learner cannot certify that this is the case, and thus it has to settle for outputting a trivial classifier.

%Towards formalizing this intractability, recall that perfect metric-fairness can always be obtained by outputting a constant classifier (which treats all individuals identically). The challenge is guaranteeing fairness together with non-trivial accuracy. A natural goal, as pioneered by \cite{DworkHPRZ12} and discussed above, is guaranteeing the best possible accuracy under the fairness constraints.
We present a natural learning problem and a metric where, even though a {\em perfectly fair and perfectly accurate} simple (linear) classifier exists, it cannot be found by any polynomial-time learning algorithm that is perfectly metric-fair. Indeed, any such algorithm can only find trivial classifiers with error rate approaching 1/2 (not much better than random guessing). The learner can tell that a particular (linear) classifier is {\em empirically} perfectly fair (and perfectly accurate). However, even though the classifier is perfectly fair on the underlying distribution, the (polynomial-time) learner cannot certify that this is the case, and thus it has to settle for outputting a trivial classifier. We note that there does exist an {\em exponential-time} perfectly metric-fair learning algorithm with a competitive accuracy guarantee,\footnote{For example, an exponential-time algorithm could learn by enumerating all possible classifiers, discarding all the ones that are not perfectly metric-fair (using a brute-force search over all pairs of individuals for each candidate classifier), and then output the most-accurate classifier among the perfectly metric-fair ones. It is important to note that this algorithm doesn't try to guarantee {\em empirical} perfect metric-fairness, which we know does not generalize. Rather, the learner has to consider the fairness behavior over all pairs of individuals.} the issue is the computational complexity of this task. In contrast, the relaxed notion of approximate metric-fairness does allow for {\em polynomial-time} relaxed-PACF learning algorithms that obtain competitive accuracy for this task (as it does for a rich class of learning problems, see Section \ref{sec:intro:algorithms}).

We present an overview of  the hard learning task and discuss its consequences below. See Section \ref{sec:hardness} and Theorem \ref{thm:hardness} for a more formal description. Since we want to argue about computational intractability, we need to make computational assumptions (in particular, if $P=NP$, then perfectly metric-fair learning would be tractable). We will  make the {\em minimal} cryptographic hardness assumption that one-way functions exist, see \cite{Goldreich2001} for further background.

\paragraph{Simplified construction.} For this sketch, we take a uniform distribution $\D$ over a domain $\X = \{ \pm 1 \}^n$. For an item (or individual) $x \in \X$, its label will be given by the linear classifier $w(x) = x_1$. Note that the linear classifier $w$ indeed is perfectly accurate.\footnote{Note that the expected margin in this distribution is small compared to the norms of the examples. This is for simplicity and readability. The full hardness result is shown (in a very similar manner) for data where the margins are large. In particular, this means that the class of predictors $H_{\sig,L}$ can achieve good accuracy with constant $L$. See Section \ref{sec:hardness}.}

To argue that fair learning is intractable, we construct two metrics $d_U$ and $d_V$ that are {\em computationally indistinguishable}: no polynomial-time algorithm can tell them apart (even given the explicit description of the metric).\footnote{More formally, we construct two {\em distribution} on metrics, such that no polynomial-time algorithm can tell whether a given metric was sampled from the first distribution or from the second. For readability, we mostly ignore this distinction in this sketch.} We construct these metrics so that $d_U$ does not allow {\em any} non-trivial accuracy, whereas $d_V$ essentially imposes no fairness constraints. Thus, $w$ is a perfectly fair and perfectly accurate classifier w.r.t. $d_V$. Now, since a polynomial-time learning algorithm $\A$ cannot tell $d_U$ and $d_V$ apart, it has to output the same (distribution on) classifiers given either of these two metrics. If $\A$, given $d_U$, outputs a classifier with non-trivial accuracy, then it violates perfect metric-fairness. Thus, when given $d_U$, $\A$ must (with high probability) output a classifier with error close to $1/2$. This remains the case even when $\A$ is given the metric $d_V$ (by indistinguishability), despite the fact perfect metric-fairness under $d_V$ allows for {\em perfect accuracy}.

We construct the metrics as follows. The metric $d_V$ gives {\em every} pair of individuals $x,x' \in \X$ distance 1. The metric $d_U$, on the other hand, partitions the items in $\X$ into disjoint pairs $(x,x')$ where the label of $x$ is $1$, the label of $x'$ is $-1$, but the distance between $x$ and $x'$ is 0.\footnote{Formally, $d_U$ is a pseudometric, since it has distinct items at distance 0. We can make $d_U$ be a true metric by replacing the distance 0 with an arbitrarily small positive quantity. The hardness result is essentially unchanged.} Thus, the metric $d_U$ assigns to each item $x \in X$ a ``hidden counterpart'' $x'$ that is {\em identical} to $x$, but has the opposite label. The distance between any two distinct elements that are not ``hidden counterparts'' is 1 (as in $d_V$). The metric $d_U$ specifies that hidden counterparts $(x,x')$ are {\em identical}, and thus any perfectly metric-fair classifier $h$ must treat them identically. Since $x$ and $x'$ have opposing labels, $h$'s average error on the pair must be $1/2$. The support of $\D$ is partitioned into disjoint hidden counterparts, and thus we conclude that $\err_\D(h) = 1/2$. Note that this is true regardless of $h$'s complexity (in particular, it also rules out improper learning).
We construct the metrics using a cryptographic pseudorandom generator (PRG), which specifies the hidden counterparts (in $d_U$) or their absence (in $d_V$). See the full version for details.

\paragraph{Discussion.} We make several remarks about the above hardness result. First, note that the data distribution is fixed, and the optimal classifier is linear and very simple: it only considers a single coordinate. This makes the hardness result sharper: without fairness, the learning task is trivial (indeed, since the classifier is fixed there is nothing to learn). It is the fairness constraint (and only the fairness constraint) that leads to intractability. The computational hardness of perfectly fair learning applies also to improper learning. Finally, the metrics for which we show hardness are arguably contrived (though we note they do obey the triangle inequality). This rules out perfectly metric-fair learners that work for {\em any} given metric. A natural direction for future work is restricting the choice of metric, which may make perfectly metric-fair learning feasible.

\subsection{Further Related Work}
\label{sec:intro:related_work}

There is a growing body of work attempting to study the question of algorithmic discrimination, particularly through the lens of machine learning. This literature is characterized by an abundance of definitions, each capturing different discrimination concerns and notions of fairness. This literature is vast and growing, and so we restrict our attention to the works most relevant to ours.

One high-level distinction can be drawn between \emph{group} and \emph{individual} notions of fairness. Group-fairness notions assume the existence of a protected attribute (e.g gender, race), which induces a partition of the instance space into some small number of groups. A fair classifier is one that achieves parity of some statistical measure across these groups. Some prominent measures include classification rates (statistical parity, see e.g ~\cite{feldman2015certifying}), calibration, and false positive or negative rates~\cite{kleinberg2016inherent,chouldechova2017fair, hardt2016equality}. It has been established that some of these notions are inherently incompatible with each other, in all but trivial cases~\cite{kleinberg2016inherent,chouldechova2017fair}. The work of~\cite{woodworth2017learning} takes a step towards incorporating the fairness notion of~\cite{hardt2016equality} into a statistical and computational theory of learning, and considers a relaxation of the fairness definition to overcome the computational intractability of the learning objective. The work of \cite{DworkIKL17} proposes an efficient framework for learning different classifiers for different groups in a fair manner.

Individual fairness~\cite{DworkHPRZ12} posits that ``similar individuals should be treated similarly''. This powerful guarantee is formalized via a Lipschitz condition (with respect to an existing task-specific similarity metric) on the classifier mapping individuals to distributions over outcomes. Recent works~\cite{joseph2016fairness, josephbetter} study different individual-level fairness guarantees in the contexts of reinforcement and online learning.
%by requiring a similar individual-fairness constraint to hold \emph{during} the process of learning of an optimal policy ~\cite{joseph2016fairness, josephbetter}.
The work of ~\cite{zemel2013learning} aims to learn an intermediate ``fair'' representation that best encodes the data while successfully obfuscating membership in a protected group. See also the more recent work \cite{BeutelCZC17}.

Several works have studied fair regression \cite{kamishima2012fairness, calders2013controlling, zafar2017fairness, berk2017convex}. The main differences in our work are a focus on metric-based individual fairness, a strong rigorous fairness guarantee, and proofs of competitive accuracy (both stated with respect to the underlying population distribution).

\section{Metric Fairness Definitions}
\label{sec:fairness_defs}

\subsection{(Perfect) Metric-Fairness}
\label{sec:perfect_fairness}

Dwork {\em et al.} \cite{DworkHPRZ12} introduced individual fairness, a similarity-based fairness notion in which a probabilistic classifier is said to be fair if it assigns similar distributions to similar individuals.

\begin{definition}[Perfect Metric-Fairness]
\label{definition:perfect_fairness}
A probabilistic classifier $h:\,\X\rightarrow\left[0,1\right]$ is said to be perfectly metric-fair w.r.t a distance metric $d:\,\X\times \X\rightarrow\left[0,1\right]$, if for every $x,x'\in \X$,
\begin{equation}
\label{eq:perfect_fairness}
\Lambda\left(h(x),h(x')\right)\leq d(x,x')
\end{equation}
\end{definition}

where $h(x)$ is interpreted as the probability $h$ will assign the label $+1$ to $x\in \X$, $\Lambda$ is a distance measure between distributions and $d$ is a task-specific distance metric that is assumed to be known in advance.  Throughout this work we take $\Lambda$ to be the statistical distance, yielding $\Lambda\left(h(x),h(x')\right)=\left|h(x)-h(x')\right|$).

In the setting considered by \cite{DworkHPRZ12}, a finite set of individuals $V$ should be assigned outcomes from a set $A$. Under the assumption that $d$ is known, they demonstrated that the problem of minimizing an arbitrary loss function $L:V\times A\rightarrow R$, subject to the individual fairness constraint  can be formulated as an LP and thus can be solved in time poly$(|A|,|V|)$.

\subsection{Approximate Metric-Fairness}
\label{sec:approximate_fairness}

We consider a learning setting in which the goal is learning a classifier $h$ that satisfies the  fairness constraint in Equation (\ref{eq:perfect_fairness}) with respect to some unknown distribution $\D$ over $\X$, after observing a finite sample $S\sim \D^m$. To this end, we introduce a metric-fairness loss function that, for a given classifier $h$ and a pair of individuals in $\X$, assigns a penalty of 1 if the fairness constraint is violated by more than a $\gamma$ additive term.

\begin{definition}
\label{definition:fairness_loss}

For a metric $d$ and $\gamma\geq0$, the metric-fairness loss on a pair $(x,x')\in\X$ is
\begin{equation}
\ell_{\gamma,d}(h,\left(x,x'\right))=\begin{cases}
1 & \Lambda\left(h(x),h(x')\right)>d(x,x')+\gamma\\
0 & \Lambda\left(h(x),h(x')\right)\leq d(x,x')+\gamma
\end{cases}
\end{equation}

\end{definition}

The overall metric-fairness loss for a hypothesis $h$ is the expected violation for a random pair according to $\D$.
\begin{definition}[Metric-Fairness Loss]
For a metric $d$ and $\gamma\geq0$,
\begin{equation}
\label{eq:true_fairness_loss}
\mathcal{L}_{\mathcal{D},d,\gamma}^{F}(h)=\mathbb{E}_{x,x'\sim\mathcal{D}}\left[\ell_{\gamma,d}(h,(x,x'))\right]
\end{equation}
\end{definition}

We go on to define the empirical fairness loss, a data-dependent quantity designed to estimate the unknown $\mathcal{L}_{\mathcal{D},d,\gamma}^{F}(h)$. To this end, we think of a sample $S\sim\D^m$ as defining a complete weighed graph, denoted $G(S)$, whose vertices are $S$ and whose edges are weighed by $w(e)=w(x_i,x_j)=d(x_i,x_j)$. Now, observe that when $S$ is sampled i.i.d from $\D$, any matching $M\subseteq G(S)$\footnote{Note that from the structure of $G(S)$, it has exactly $m$ matchings, each of size $\frac{m-1}{2}$ (w.l.o.g, we assume $m$ is odd). } is an i.i.d sample from $\D\times\D$. We now define the empirical loss by replacing the expectation over $\D$ in Equation (\ref{eq:true_fairness_loss}) with the expectation over some matching $M\subseteq G(S)$.

\begin{definition}[Empirical Metric-Fairness loss]
\begin{equation}
\mathcal{L}_{\mathcal{S},d,\gamma}^{F}(h)=\frac{2}{m-1}\cdot\sum_{(x,x')\in M(S)}\left[\ell_{\gamma,d}(h,(x,x'))\right]
\end{equation}
\end{definition}

Finally, we will say that a classifier $h$ is $\left(\alpha,\gamma\right)$-fair w.r.t $\D$ (respectively, $S$)  and $d$ if its respective metric-fairness loss is at most $\alpha$.

\begin{definition}[$\left(\alpha,\gamma\right)$-Metric-Fairness]
\label{def:alpha_fairness}
A probabilistic classifier $h:\,\X\rightarrow\left[0,1\right]$ is said to be $\left(\alpha,\gamma\right)$-fair w.r.t a metric $d$ and $\D$ (respectively, $S$) if $\mathcal{L}_{\mathcal{D},d,\gamma}^{F}(h)\leq\alpha$
(respectively, $\mathcal{L}_{\mathcal{S},d,\gamma}^{F}(h)\leq\alpha$).
\end{definition}

\paragraph{Notation} When $S$ and $d$ are  clear from context, we will use the more succinct notation $\mathcal{L}^{F}_{\gamma}(h)$ for the true fairness loss and $\widehat{\mathcal{L}}^{F}_{\gamma}(h)$ for the empirical fairness loss. When dealing with a hypothesis class $\H$, we use $\H^{\alpha,\gamma}\subseteq \H$ to denote all the $(\alpha,\gamma)$-fair hypotheses in $\H$ (w.r.t $\D$), and $\widehat{H}^{\alpha,\gamma}\subseteq \H$ for those which are  $(\alpha,\gamma)$-fair w.r.t $S$.

\subsection{Approximate Metric Fairness: Interpretation}
\label{sec:AMF_interpretation}

An $\alpha$-fair classifier (for $\alpha>0$) no longer holds any guarantee for any single individual. To interpret the guarantee it \textit{does}  give, we consider the following definition.

\begin{definition}
\label{def:interpretable_fairness}
A probabilistic classifier $h:\,\X\rightarrow\left[0,1\right]$ is said to be $\left(\alpha_{1},\alpha_{2};\gamma\right)$ metric-fair w.r.t $d,\D$ if
\begin{equation}
Pr_{x\sim \D}\left[Pr_{x'\sim \D}\left[\Lambda\left(h(x),h(x')\right)>d(x,x')+\gamma\right]>\alpha_{2}\right]\leq\alpha_{1}
\end{equation}
\end{definition}

Definition~\ref{def:interpretable_fairness} is very similar to~\ref{def:alpha_fairness} but it lends itself to a more intuitive interpretation of fairness for groups. Informally, we will say that an individual feels $\alpha_{2}$-discriminated against by $h$ if the proportion of individuals with whom his constraint is violated (think: individuals who are equally qualified to him but receive different treatment) exceeds $\alpha_{2}$; now, $\left(\alpha_{1},\alpha_{2}\right)$-fairness  ensures that the proportion of individuals who find $h$ to be $\alpha_{2}$-discriminatory does not exceed $\alpha_{1}$. Hence, this is a guarantee for groups: an $\left(\alpha_{1},\alpha_{2}\right)$-fair classifier cannot cause an entire group of fractional mass $\alpha_{1}$ to be discriminated against. The strength of this guarantee is that it holds for \textit{any} such group (even for those formed “ex-ante”). In this sense, $\left(\alpha_{1},\alpha_{2}\right)$-fairness represents a middle-ground between the strict notion of individual fairness and the loose notions of group-fairness.

Finally, we show that the two definitions are related: any $\alpha$-fair classifier is also $\left(\alpha_{1},\alpha_{2}\right)$-fair, for every $\alpha_{1},\alpha_{2}$ for which $\alpha_{1}\cdot\alpha_{2}\geq\alpha$. This demonstrates that optimizing for accuracy under an $\alpha$-fairness constraint is a flexible way of achieving interpretable fairness guarantees for a range of desired $\alpha_{1},\alpha_{2}$ values.

\begin{claim}
\label{cl:interpretable_fairness}
For every $\alpha,\gamma\in(0,1)$, and $\alpha_{1},\alpha_{2}\in(0,1)$ for which $\alpha_{1}\cdot\alpha_{2}\geq\alpha$, if $h$ is $\left(\alpha,\gamma\right)$-fair then it is also $\left(\alpha_{1},\alpha_{2};\gamma\right)$-fair.
\end{claim}

\begin{proofof}{Claim~\ref{cl:interpretable_fairness}}
For simplicity, we define the following indicator function,
\[
\mathds{1}_{\gamma, h}(x,x')=\begin{cases}
1 & \Lambda\left(h(x),h(x')\right)>d(x,x')+\gamma\\
0 & o.w
\end{cases}
\]

Let $\alpha,\gamma,\alpha_{1},\alpha_{2}\in(0,1)$
such that $\alpha_{1}\cdot\alpha_{2}\geq\alpha$, and assume that
$h$ is $(\alpha,\gamma)$-fair w.r.t $\mathcal{D}$. Assume for contradiction
that $h$ is not $\left(\alpha_{1},\alpha_{2};\gamma\right)$-fair
w.r.t $\mathcal{D}$. That means that
\[
Pr_{x\sim\mathcal{D}}\left[Pr_{x'\sim\mathcal{D}}\left[\mathds{1}_{\gamma, h}(x,x')=1\right]>\alpha_{2}\right]>\alpha_{1}
\]
If we denote the subset of ``$\alpha_{2}$-discriminated'' individuals
as $B$,
\[
B\triangleq\left\{ x\in\mathcal{X}\,:\,Pr_{x'\sim\mathcal{D}}\left[\mathds{1}_{\gamma, h}(x,x')=1\right]>\alpha_{2}\right\}
\]
then the assumption is equivalent to $Pr_{x\sim\mathcal{D}}\left[x\in S\right]>\alpha_{1}$. We now obtain:
\begin{align*}
\alpha & \geq Pr_{x,x'\sim\mathcal{D}}\left[\mathds{1}_{\gamma, h}(x,x')=1\right]\\
 & =Pr_{x\sim\mathcal{D}}\left[x\in S\right]\cdot Pr_{x'\sim\mathcal{D}}\left[\mathds{1}_{\gamma, h}(x,x')=1\vert x\in S\right]+\underbrace{Pr_{x\sim\mathcal{D}}\left[x\notin S\right]\cdot Pr_{x'\sim\mathcal{D}}\left[\mathds{1}_{\gamma, h}(x,x')=1\vert x\notin S\right]}_{\geq0}\\
 & \geq Pr_{x\sim\mathcal{D}}\left[x\in S\right]\cdot Pr_{x'\sim\mathcal{D}}\left[\mathds{1}_{\gamma, h}(x,x')=1\vert x\in S\right]\\
 & >\alpha_{1}\cdot\alpha_{2}
\end{align*}
where the first transition is from the assumption that $h$ is $\left(\alpha,\gamma\right)$-fair,
and the final transition is from the assumption that $h$ is not $\left(\alpha_{1},\alpha_{2};\gamma\right)$-fair.
We therefore have that $\alpha>\alpha_{1}\cdot\alpha_{2}$, which
contradicts our assumption.
\end{proofof}

\subsection{$\ell_{1}$-Metric-Fairness}
\label{sec:L1}

In this work we focus on approximate metric-fairness and the $\left(\alpha,\gamma\right)$ metric fairness loss (Definition \ref{definition:fairness_loss}). We find that this notion provides appealing and \emph{interpretable} protections from discrimination: as discussed above, for small enough $\alpha,\gamma$, every sufficiently large group is protected from blatant discrimination (see Section \ref{sec:AMF_interpretation}). However, in turning to design efficient metric-fair learning algorithms, working directly with this definition presents difficulties (see Section \ref{sec:learning-linear}). In particular, the ``0/1'' nature of the metric fairness loss means that the set $\widehat{H}^{\alpha,\gamma}$ is not a convex set. Trying to learn an empirically  $(\alpha, \gamma)$ metric-fair that optimizes accuracy is a non-convex optimization problem, and it isn't clear how to optimize using convex-optimization tools.

In light of this difficulty, we introduce a different metric-fairness loss definition. It overcomes the non-convexity by replacing the bound on the expected \emph{number} of fairness violations with a bound on the expected \emph{sum} of the fairness violations.
\begin{definition}[$\ell_{1}$ MF loss]
\label{def:l1_fairness_loss}
For a metric $d$, the $\ell_{1}$ metric-fairness loss on a pair $\left(x,x'\right)\in\mathcal{X}$ is \begin{equation*}
\ell_{d}^{1}\left(h,(x,x')\right)=\max\left(0,\left|h(x)-h(x')\right|-d(x,x')\right)
\end{equation*}
\end{definition}
Similarly to the regular metric-fairness loss, the loss for a hypothesis $h$ is the expected violation for a random pair according to $\mathcal{D}$, and the empirical loss replaces the expectation over $\mathcal{D}$ with the expectation over some matching $M\subseteq G(S)$.
\begin{definition}[$\tau$ $\ell_{1}$-Metric-Fairness]
\label{def:L1_fairness}
A probabilistic classifier $h$ is said to be $\tau$ $\ell_{1}$-metric-fair w.r.t a metric $d$ and w.r.t $\mathcal{D}$ (respectively, $S$) if its respective $\ell_{1}$ MF loss is bounded by $\tau$.
\end{definition}
When $S,\D, d$ are clear from context we use the notation $H_{\ell_{1}}^{\tau}$ (respectively, $\widehat{H}_{\ell_{1}}^{\tau}$) for the subset of hypotheses from $H$ which are $\tau$ $\ell_{1}$-MF w.r.t $\mathcal{D}$ (respectively, $S$). We use $H_{\ell_{0}}^{\alpha,\gamma}$ (previously $H^{\alpha,\gamma}$) to emphasize that fairness is calculated w.r.t the standard MF loss.

The main advantage of $\ell_1$ metric-fairness is that it induces convex constraints and tractable optimization problems. We note that $\ell_1$ metric-fairness also generalizes from a sample (indeed, in this case the fairness loss is Lipschitz and thus it's easier to prove generalization bounds). The main disadvantage of $\ell_1$ metric-fairness is in interpreting the guarantee: it is less immediately obvious how bounding the sum of ``fairness deviations'' translates into protections for groups or for individuals. Nonetheless, we show that $\tau$ $\ell_{1}$-metric-fairness implies $(\tau/\gamma,\gamma)$ approximate metric fairness for every $\gamma > 0$. Thus, when we optimize over the set of $\tau$  $\ell_1$-MF predictors, we are guaranteed to output an approximately metric-fair predictor. Moreover, since approximate MF also implies $\ell_1$-MF, any solution to the ($\ell_1$) optimization problem will also be competitive with a certain class of approximately metric-fair classifiers. 

The connection between approximate and $\ell_1$ metric fairness is quantified in Lemma \ref{lemma:l1_l0} below. We note that there is a gap between these upper and lower bounds. This is the ``price'' we pay for relaxing from approximate to $\ell_1$ metric fairness. We also note that in proving that $\ell_1$-MF implies approximate metric-fairness, it is essential to use a non-zero additive $\gamma$ violation term.

\begin{lemma}
\label{lemma:l1_l0}
For every sample $S$, matching $M(S)$ and every $\tau,\gamma\in(0,1)$,
$\widehat{H}_{\ell_{0}}^{\tau-\gamma,\gamma}\subseteq\widehat{H}_{\ell_{1}}^{\tau}\subseteq\widehat{H}_{\ell_{0}}^{\frac{\tau}{\gamma},\gamma}$.
\end{lemma}
\begin{proofof}{Lemma~\ref{lemma:l1_l0}}
We begin by defining the {\em induced violation vector} of a classifier $h \in H$. For a sample $S$, a matching $M\subseteq G(S)$ and a value $\gamma\in(0,1)$, the induced violation vector $\xi_{\gamma}(h)\in\left\{ 0,1\right\} ^{\left|M\right|}$ is defined as:
\begin{equation*}
\left[\xi_{\gamma}(h)\right]_{i}=\max\left\{0,\left|h(\mathbf{x}) - h(\mathbf{x}') \right| - d(\mathbf{x},\mathbf{x}') - \gamma\right\}
\end{equation*}
where $(\mathbf{x},\mathbf{x'})$ is the $i$-th edge in the matching $M$.

We first prove that $\widehat{H}_{\ell_{1}}^{\tau}\subseteq\widehat{H}_{\ell_{0}}^{\frac{\tau}{\gamma},\gamma}$.  If  $h\in\widehat{H}_{\ell_{1}}^{\tau}$, then this means that $\|\xi_{0}(h)\|_{1}\leq\tau$. Assume for contradiction that for some $\gamma > 0$ we have $\|\xi_{\gamma}(h)\|_{0}>\frac{\tau}{\gamma}$.
But this implies that $\|\xi_{0}(h)\|_{1}\geq \|\xi_{\gamma}(h)\|_{0} \cdot\gamma > \frac{\tau}{\gamma}\cdot \gamma = \tau$, which is a contradiction. We therefore have that $\|\xi_{\gamma}(h)\|_{0}\leq\frac{\tau}{\gamma}$ from which we conclude that $h\in\widehat{H}_{\ell_{0}}^{\frac{\tau}{\gamma},\gamma}$.

Next, we prove that $\widehat{H}_{\ell_{0}}^{\tau-\gamma,\gamma}\subseteq\widehat{H}_{\ell_{1}}^{\tau}$. To do so, we'll prove that $\widehat{H}_{\ell_{0}}^{\tau-\gamma,\gamma}\subseteq\widehat{H}_{\ell_{0}}^{\frac{\tau-\gamma}{1-\gamma},\gamma}\subseteq\widehat{H}_{\ell_{1}}^{\tau}$.
Observe that for every $\tau,\gamma \in (0,1)$ we have that $\widehat{H}_{\ell_{0}}^{\tau-\gamma,\gamma}\subseteq\widehat{H}_{\ell_{0}}^{\frac{\tau-\gamma}{1-\gamma},\gamma}$
because $\frac{\tau-\gamma}{1-\gamma}\geq\tau-\gamma$. To see
that $\widehat{H}_{\ell_{0}}^{\frac{\tau-\gamma}{1-\gamma},\gamma}\subseteq\widehat{H}_{\ell_{1}}^{\tau}$,
recall that $h\in\widehat{H}_{\ell_{0}}^{\frac{\tau-\gamma}{1-\gamma},\gamma}$
implies that $\|\xi_{\gamma}(h)\|_{0}\leq\frac{\tau-\gamma}{1-\gamma}$.
Thus:
\begin{align*}
\|\xi_{0}(h)\|_{1} & \leq\frac{\tau-\gamma}{1-\gamma}\cdot1+\left(1-\frac{\tau-\gamma}{1-\gamma}\right)\cdot\gamma=\tau
\end{align*}

which implies $h\in\widehat{H}_{\ell_{1}}^{\tau}$, as
required.
\end{proofof}

\subsection{Generalization}
\label{sec:fairness_generalization}

A key issue in learning theory is that of generalization: to what extent is a classifier that is accurate on a finite sample $S\sim\D^m$ also guaranteed to be accurate w.r.t the underlying distribution? In this section, we work to develop similar generalization bounds for our metric-fairness loss function. Proving that fairness can generalize well is a crucial component in our analysis -  it effectively rules out the possibility of creating a ``false facade'' of fairness (i.e, a classifier that only appears fair on a sample, but is not fair w.r.t new individuals).

The generalization bounds will be based on proving uniform convergence of the empirical estimates (in our case, $\widehat{\mathcal{L}}^{F}_{\gamma}(h)$) to the fairness loss, simultaneously for every $h\in \H$, in terms of the Rademacher complexity of the hypotheses class $\H$. Rademacher complexity differs from celebrated VC-dimension complexity measure in three aspects: first, it is defined for any class of real-valued functions (making it suitable for our setting of learning probabilistic classifiers);  second, it is data-dependent and can be measured from finite samples; third, it often results in tighter uniform convergence bounds (see, e.g, ~\cite{koltchinskii2002empirical}).

\begin{definition}[Rademacher complexity]
\label{def:Rademacher}
Let $Z$ be an input space, $\D$ a distribution on $Z$, and $\mathcal{F}$ a real-valued function class defined on $Z$. The empirical Rademacher complexity of $\mathcal{F}$ with respect to a sample $S=\left\{ z_{1}\dots z_{m}\right\}$ is the following random variable:
\begin{equation}
\widehat{\mathcal{R}}_{m}(\mathcal{F})=\mathbb{E}_{\sigma}\left[\sup_{f\in\mathcal{F}}\left(\frac{1}{m}\sum_{i=1}^{m}\sigma_{i}f(z_{i})\right)\right]
\end{equation}

The expectation is taken over $\sigma=\left(\sigma_1,\dots,\sigma_m\right)$ where the $\sigma_i$'s are independent uniformly random variables taking values in $\{{\pm1}\}$. The Rademacher complexity of $\mathcal{F}$ is defined as the expectation of $\widehat{\mathcal{R}}_{m}(\mathcal{F})$ over all samples of size $m$:

\begin{equation}
\mathcal{R}_{m}(\mathcal{F})=\mathbb{E}_{S\sim\mathcal{D}^{m}}\left[\widehat{\mathcal{R}}_{m}(\mathcal{F})\right]
\end{equation}

\end{definition}

In our case, the fact that our fairness loss $\ell_{\gamma,d}$ is a 0-1 style-loss means that for infinite hypothesis classes, the generalization argument is involved and we incur an extra approximation term in $\gamma$.

\begin{theorem}[Rademacher-Based Uniform Convergence of the Metric-Fairness Loss]
\label{thm:rad_generalization}
Let $\mathcal{H}$ be a hypotheses class with Rademacher complexity $R_m(\mathcal{H})$. For every $\delta,\gamma\in(0,1)$, every $G\geq1$ and every $m\geq0$ (w.l.o.g assume $m$ is odd), with probability at least $1-\delta$ over an i.i.d sample $S\sim\D^m$, simultaneously for every $h\in\mathcal{H}$:
\begin{equation}
\mathcal{L}_{\gamma+\frac{1}{G}}^{F}(h)-\Delta_{m}	\leq\widehat{\mathcal{L}}_{\gamma}^{F}(h)\leq\mathcal{L}_{\gamma-\frac{1}{G}}^{F}(h)+\Delta_{m}
\end{equation}

where $\Delta_{m}=2G\cdot\left(4\widehat{R}_{\frac{m-1}{2}}\left(\mathcal{H}\right)+\frac{4+17\sqrt{\ln(4/\delta)}}{\sqrt{m-1}}\right)$.
\end{theorem}

Before proving the theorem, we note that an immediate corollary is that our metric-fairness loss is capable of generalizing well for any hypothesis class that has a small Rademacher complexity. In particular, we will be using the following result for the class of linear classifiers with norm bounded by $C$ in a RHKS $\mathbb{V}$ whose inner products are implemented by a kernel $K$ (i.e, exists a mapping $\psi_K:\,\mathcal{X}\rightarrow\mathbb{V}$ such that $\left\langle \psi(\boldsymbol{\mathbf{x}}),\psi(\mathbf{x}')\right\rangle =K(\mathbf{x},\mathbf{x}')$):
\[
H_{\psi,C}\overset{{\scriptscriptstyle {\scriptstyle \text{def}}}}{=}\left\{ \mathbf{x}\mapsto\left\langle \mathbf{v},\psi(\mathbf{x})\right\rangle :\,\,\,\,\mathbf{v}\in\mathbb{V},\,\|\mathbf{v}\|\leq C\right\}
\]

\begin{corollary}
\label{cor:linear_generalization}
Let $H_{\psi,C}$ as above, for any $C\geq0$ and kernel $K$. For every $\delta, \gamma\in(0,1)$, every $G\geq1$ and every $m\geq0$, w.p at least $1-\delta$, simultaneously for every $h\in H_{\psi,C}$
\begin{equation}
\mathcal{L}_{\gamma+\frac{1}{G}}^{F}(h)-\Delta_{m}\leq\widehat{\mathcal{L}}_{\gamma}^{F}(h)\leq\mathcal{L}_{\gamma-\frac{1}{G}}^{F}(h)+\Delta_{m}
\end{equation}

where $M=\sup K(\mathbf{x},\mathbf{x}')$ and $\Delta_{m}=2G\cdot\frac{4+4\sqrt{2}\sqrt{CM}+17\sqrt{\ln(4/\delta)}}{\sqrt{m-1}}$.

\end{corollary}

\begin{proofof}{Corollary~\ref{cor:linear_generalization}}
The proof follows from Theorem~\ref{thm:rad_generalization} and the fact that
the Rademacher complexity ${R}_m{(H_{\psi,C})}$ is bounded by $\sqrt{\frac{C\cdot M}{m}}$ (see~\cite{kakade2009complexity}).

\end{proofof}

To set the stage for proving Theorem~\ref{thm:rad_generalization}, we state some useful properties of the Rademacher complexity notion, see~\cite{bartlett2002rademacher}.

\begin{theorem}[Two-sided Rademacher-based uniform convergence]
\label{thm:known_rad_uniform_conv}
Consider a set of functions $\mathcal{F}$ mapping $Z$ to $\left[0,1\right]$. For every $\delta>0$, with probability at least $1-\delta$ over a random draw of a sample $S$ of size $m$, every $f\in F$ satisfies
\begin{align}
\label{eq:known_rad_uniform_conv1}
 & \left|\mathbb{E}_{D}\left[f(z)\right]-\mathbb{\widehat{E}}\left[f(z)\right]\right|\leq2R_{m}(F)+\sqrt{\frac{\log\frac{2}{\delta}}{2m}}\\
 \label{eq:known_rad_uniform_conv2}
 & \left|\mathbb{E}_{D}\left[f(z)\right]-\mathbb{\widehat{E}}\left[f(z)\right]\right|\leq2\widehat{R}_{m}(F)+3\sqrt{\frac{\log\frac{4}{\delta}}{2m}}
\end{align}
\end{theorem}

\begin{lemma}
\label{lemma:rad_facts}\ Let $F$ be a class of real functions. Then, for any sample $S$ of size $m$:
\begin{enumerate}
\item For any function $h$, $\widehat{R}_{m}(F+h)\leq\widehat{R}_{m}(F)+2\sqrt{\frac{\mathbb{\widehat{E}}[h^{2}]}{m}}$.
\item If $A\,:\,\mathbb{R}\rightarrow\mathbb{R}$ is $L$-Lipschitz and satisfies $A(0)=0$, then $\widehat{R}_{m}(A\circ F)\leq2L\cdot\widehat{R}_{m}(F)$.
\item For every $\delta\in(0,1)$, w.p at least $1-\delta$ over the choice of $S$,
\begin{equation*}
-2\sqrt{\frac{\log\frac{2}{\delta}}{2m}}\leq R_{m}(F)-\widehat{R}_{m}(F)\leq2\sqrt{\frac{\log\frac{2}{\delta}}{2m}}
\end{equation*}
\end{enumerate}
\end{lemma}

\begin{proofof}{Theorem~\ref{thm:rad_generalization}}

Consider the input space $Z=\mathcal{X}\times\mathcal{X}$ and consider
the functions $\mathcal{F}=\left\{ \ell_{\gamma}^{h}:\,\,h\in\mathcal{H}\right\} $,
where $\ell_{\gamma}^{h}:\,Z\rightarrow\mathbb{R}$ is the fairness
loss defined as
\[
\ell_{\gamma}^{h}(z)=\ell_{h}(x,x')	=\mathds{1}\left[\left|h(x)-h(x')\right|>d(x,x')+\gamma\right]
	\triangleq\begin{cases}
1 & \left|h(x)-h(x')\right|>d(x,x')+\gamma\\
0 & o.w
\end{cases}
	\]

For a given sample $S\sim\mathcal{D}^{m}$ (w.l.o.g, assume $m$ is
odd), let $M=\left\{ z_{1},\dots,z_{\frac{m-1}{2}}\right\} \subseteq Z$
be any matching in the graph induced by $S$. Observe that $M$ is
indeed an i.i.d sample from $\mathcal{D}\times\mathcal{D}$ and recall that $\left|M\right|=\frac{m-1}{2}\triangleq \tilde{m}$. For the
sake of simplicity, we hereby denote $\widehat{R}_{\tilde{m}}(F)$
as the empirical Rademacher complexity with respect to $M$ induced
by a random sample $S$.

Denote the threshold function at $\gamma$:
\[
\text{\ensuremath{\sigma}\ensuremath{\ensuremath{_{\gamma}}}\ensuremath{(x)=\begin{cases}
1 & x>\gamma\\
0 & x\leq\gamma
\end{cases}}}
\]

Hence,
\begin{align*}
\mathcal{F} & =\text{\ensuremath{\sigma}\ensuremath{\ensuremath{_{\gamma}}}\,\ensuremath{\circ\,\mathcal{G}}}\\
\mathcal{G} & =\left\{ (x,x')\mapsto\left|h(x)-h(x')\right|-d(x,x')\right\} _{h\,\in\,\mathcal{H}}
\end{align*}

Observe that $\mathcal{G}$ can be further decomposed as
\[
\mathcal{G}=\text{abs\,\ensuremath{\circ}}\,\mathcal{H}'+f
\]

where $\mathcal{H}'=\left\{ (x,x')\mapsto h(x)-h(x')\right\}_{h\,\in\,\mathcal{H}} \triangleq\left\{ g_{h}\right\}_{h\,\in\,\mathcal{H}} $,
abs$(\cdot)$ is the absolute value function (which is 1-Lipschitz),
and $f=f(x,x')=-d(x,x')$.

\begin{claim}
\label{claim:rad_H'_H}
$R_{\widetilde{m}}(\mathcal{H}')\leq2R_{\widetilde{m}}(\mathcal{H})$.
\end{claim}

\begin{proofof}{Claim~\ref{claim:rad_H'_H}}
Denote $M(S)=\left\{ z_{1},\dots,z_{\widetilde{m}}\right\} ,$where
$z_{i}=\left(x_{i}^{1},x_{i}^{2}\right)$. By definition,
\begin{align*}
R_{\widetilde{m}}(\mathcal{H}') & =\mathbb{E}_{S\sim\mathcal{D}^{m}}\left[\widehat{R}_{\widetilde{m}}(\mathcal{H}')\right]\\
 & =\mathbb{E}_{S,\sigma}\left[\sup_{g_{h}\in\mathcal{H}'}\left(\frac{1}{\tilde{m}}\sum_{i=1}^{m}\sigma_{i}g_{h}(z_{i})\right)\right]\\
 & =\mathbb{E}_{S,\sigma}\left[\sup_{g_{h}\in\mathcal{H}'}\left(\frac{1}{\tilde{m}}\sum_{i=1}^{m}\sigma_{i}\left(h(x_{i}^{1})-h(x_{i}^{2})\right)\right)\right]\\
 & \leq\mathbb{E}_{S,\sigma}\left[\sup_{g_{h}\in\mathcal{H}'}\left(\frac{1}{\tilde{m}}\sum_{i=1}^{m}\sigma_{i}h(x_{i}^{1})\right)+\sup_{g_{h}\in\mathcal{H}'}\left(\frac{1}{\tilde{m}}\sum_{i=1}^{m}-\sigma_{i}h(x_{i}^{2})\right)\right]\\
 & =\mathbb{E}_{S,\sigma}\left[\sup_{h\in\mathcal{H}}\left(\frac{1}{\tilde{m}}\sum_{i=1}^{m}\sigma_{i}h(x_{i}^{1})\right)+\sup_{h\in\mathcal{H}}\left(\frac{1}{\tilde{m}}\sum_{i=1}^{m}-\sigma_{i}h(x_{i}^{2})\right)\right]\\
(\star) & =\mathbb{E}_{S,\sigma}\left[\sup_{h\in\mathcal{H}}\left(\frac{1}{\tilde{m}}\sum_{i=1}^{m}\sigma_{i}h(x_{i}^{1})\right)\right]+\mathbb{E}_{S,\sigma}\left[\sup_{h\in\mathcal{H}}\left(\frac{1}{\tilde{m}}\sum_{i=1}^{m}\sigma_{i}h(x_{i}^{2})\right)\right]\\
 & =2R_{\widetilde{m}}(\mathcal{H})
\end{align*}
where the transition marked by $(\star)$ is due to the fact that
negating a Rademacher variable does not change its distribution.
\end{proofof}

\begin{claim}
\label{claim:rad_G_H}
$R_{\widetilde{m}}(\mathcal{G})  \leq4R_{\widetilde{m}}(\mathcal{H})+\frac{2}{\sqrt{\widetilde{m}}}$
\end{claim}

\begin{proofof}{Claim~\ref{claim:rad_G_H}}
\begin{align*}
R_{\widetilde{m}}(\mathcal{G}) & \equiv\mathbb{E}_{S\sim\mathcal{D}^{m}}\left[\widehat{R}_{\widetilde{m}}(\mathcal{G})\right]\\
 & \equiv\mathbb{E}_{S\sim\mathcal{D}^{m}}\left[\widehat{R}_{\widetilde{m}}(\text{abs\,\ensuremath{\circ}}\,\mathcal{H}'+f)\right]\\
\left(\text{Fact 1 in~\ref{lemma:rad_facts}}\right) & \leq\mathbb{E}_{S\sim\mathcal{D}^{m}}\left[\widehat{R}_{\widetilde{m}}(\text{abs\,\ensuremath{\circ}}\,\mathcal{H}')+2\sqrt{\frac{\mathbb{\widehat{E}}[f^{2}]}{\widetilde{m}}}\right]\\
 & \leq\mathbb{E}_{S\sim\mathcal{D}^{m}}\left[\widehat{R}_{\widetilde{m}}(\text{abs\,\ensuremath{\circ}}\,\mathcal{H}')\right]+\frac{2}{\sqrt{\widetilde{m}}}\\
\left(\text{Fact 2 in~\ref{lemma:rad_facts}}\right) & \leq\mathbb{E}_{S\sim\mathcal{D}^{m}}\left[2\widehat{R}_{\widetilde{m}}(\mathcal{H}')\right]+\frac{2}{\sqrt{\widetilde{m}}}\\
 & =2R_{\widetilde{m}}(\mathcal{H}')+\frac{2}{\sqrt{\widetilde{m}}}\\
\left(\text{Claim~\ref{claim:rad_H'_H}}\right) & \leq4R_{\widetilde{m}}(\mathcal{H})+\frac{2}{\sqrt{\widetilde{m}}}
\end{align*}
\end{proofof} \\
Now, had the threshold function $\sigma\ensuremath{_{\gamma}}$ been
Lipschitz, we could again use Fact 2 in Lemma~\ref{lemma:rad_facts} to finish the proof. Unfortunately,
$\sigma\ensuremath{_{\gamma}}$ is not Lipschitz. We therefore instead approximate it using a piecewise
linear function with Lipschitz constant $G$, which we denote with
$\tau_{\gamma}^{G}$:
\[
\tau_{\gamma}^{G}(x)=\begin{cases}
0 & x\leq\gamma\\
G\left(x-\gamma\right) & \gamma<x<\gamma+\frac{1}{G}\\
1 & x\geq\gamma+\frac{1}{G}
\end{cases}
\]

\begin{proposition}
\label{prop:L_sigma_tau}
For every $G\geq0$, every $\gamma\in(0,1)$ and every function $h$,
\[
\mathcal{L}_{\sigma}(h) \leq \mathcal{L}_{\gamma}^{\tau^{G}}(h)
\leq \mathcal{L}_{\gamma+\frac{1}{G}}^{\sigma}(h)
\]
\end{proposition}

\begin{proofof}{Proposition~\ref{prop:L_sigma_tau}}
The proof follows directly from the fact that from the construction of $\tau_{\gamma}^{G}$, it holds that for every $u$, $\sigma_{\gamma+\frac{1}{G}}(u)\leq \tau_{\gamma}^{G}(u)\leq  \sigma_{\gamma}(u)$.
\end{proofof} \\

Now, letting
\[
\widetilde{\mathcal{F}} = \tau_{\gamma}^{G}\,\circ\, \mathcal{G}
\]

We can use the Lipschitzness of $\tau_{\gamma}^{G}$ to obtain
\begin{align*}
R_{\widetilde{m}}(\mathcal{\widetilde{\mathcal{F}}}) & =\mathbb{E}_{S\sim\mathcal{D}^{m}}\left[\widehat{R}_{\widetilde{m}}(\mathcal{\mathcal{\widetilde{\mathcal{F}}}})\right]\\
\left(\text{Fact 2 in Lemma~\ref{lemma:rad_facts}}\right) & \leq2G\cdot\mathbb{E}_{S\sim\mathcal{D}^{m}}\left[\widehat{R}_{\widetilde{m}}(\mathcal{G})\right]\\
 & =2G\cdot R(\mathcal{G})\\
\left(\text{Claim~\ref{claim:rad_G_H}}\right) & \leq2G\left[4R_{\widetilde{m}}(\mathcal{H})+\frac{2}{\sqrt{\widetilde{m}}}\right]\\
 & =8G\cdot R_{\widetilde{m}}(\mathcal{H})+\frac{4G}{\sqrt{\widetilde{m}}}
\end{align*}

Define the following loss notations:
\[
\ell_{\gamma}^{\sigma}(h;x,x')=\sigma_{\gamma}\left(\left|h(x)-h(x')\right|-d(x,x')\right)
\]
\[
\ell_{\gamma}^{\tau^{G}}(h;x,x')=\tau_{\gamma}^{G}\left(\left|h(x)-h(x')\right|-d(x,x')\right)
\]

and let $\mathcal{L}_{\gamma}^{\tau^{G}}(h)$ and $\widehat{\mathcal{L}}_{\gamma}^{\tau^{G}}(h)$ (respectively, $\mathcal{L}_{\gamma}^{\sigma}(h)$ and $\widehat{\mathcal{L}}_{\gamma}^{\sigma}(h)$)  denote the average with respect to $\D$ and $M(S)$ of $\ell_{\gamma}^{\tau^{G}}(h;x,x')$ (respectively, $\ell_{\gamma}^{\sigma}(h;x,x')$).

By plugging $\mathcal{\widetilde{\mathcal{F}}}$ into Theorem~\ref{thm:known_rad_uniform_conv} (Equation~\ref{eq:known_rad_uniform_conv1}), we'd have that with probability at least $1-\frac{\delta}{2}$
over random draws of samples of size $\widetilde{m}$, every $h\in H$
satisfies
\begin{align*}
\left|\mathcal{L}_{\gamma}^{\tau^{G}}(h)-\widehat{\mathcal{L}}_{\gamma}^{\tau^{G}}(h)\right| & \leq2R_{\widetilde{m}}(\mathcal{\widetilde{\mathcal{F}}})+\sqrt{\frac{\ln(4/\delta)}{2\widetilde{m}}}\\
 & \leq8G\cdot R_{\widetilde{m}}(\mathcal{H})+\frac{4G}{\sqrt{\widetilde{m}}}+\sqrt{\frac{\ln(4/\delta)}{2\widetilde{m}}}
\end{align*}

Since we eventually want a data-dependent bound, we'll Fact 3 in Lemma~\ref{lemma:rad_facts} to deduce that w.p at least $1-\frac{\delta}{2}$,
\begin{equation}
\label{eq:emp_rad_to_rad}
-2\sqrt{\frac{\ln(4/\delta)}{2\widetilde{m}}}\leq R_{\widetilde{m}}(\mathcal{H})-\widehat{R}_{\widetilde{m}}(\mathcal{H})\leq2\sqrt{\frac{\ln(4/\delta)}{2\widetilde{m}}}
\end{equation}

Now, note that
\begin{align*}
\widehat{\mathcal{L}}_{\gamma}^{F}(h) & \triangleq\widehat{\mathcal{L}}_{\gamma}^{\sigma}(h)\\
 & \geq\widehat{\mathcal{L}}_{\gamma}^{\tau^{G}}(h)\\
 & \geq\mathcal{L}_{\gamma}^{\tau^{G}}(h)-\left(8G\cdot R_{\widetilde{m}}(\mathcal{H})+\frac{4G}{\sqrt{\widetilde{m}}}+\sqrt{\frac{\ln(2/\delta)}{2\widetilde{m}}}\right)\\
\left(\text{Equation (\ref{eq:emp_rad_to_rad}), w.p $1-\tfrac{\delta}{2}$}\right) & \geq\mathcal{L}_{\gamma}^{\tau^{G}}(h)-\left(8G\widehat{R}_{\widetilde{m}}(\mathcal{H})+\frac{4G}{\sqrt{\widetilde{m}}} + 16G\cdot\sqrt{\frac{\ln\left(4/\delta\right)}{2\widetilde{m}}}+\sqrt{\frac{\ln(2/\delta)}{2\widetilde{m}}} \right)\\
& \geq \mathcal{L}_{\gamma}^{\tau^{G}}(h)-\left(8G\widehat{R}_{\widetilde{m}}(\mathcal{H})+\frac{4G+17G\cdot\sqrt{\ln(4/\delta)}}{\sqrt{\widetilde{m}}} \right) \\
\left(\text{Proposition~\ref{prop:L_sigma_tau}}\right) & \geq\mathcal{L}_{\gamma+\frac{1}{G}}^{\sigma}(h)-G\cdot \left(8\widehat{R}_{\widetilde{m}}(\mathcal{H}) + \frac{4+17\sqrt{\ln(4/\delta)}}{\sqrt{\widetilde{m}}} \right)\\
 & \triangleq\mathcal{L}_{\gamma+\frac{1}{G}}^{F}(h)-G\cdot \left(8\widehat{R}_{\widetilde{m}}(\mathcal{H}) + \frac{4+17\sqrt{\ln(4/\delta)}}{\sqrt{\widetilde{m}}} \right)
\end{align*}

Similarly,
\begin{equation*}
\widehat{\mathcal{L}}_{\gamma}^{F}(h)\leq\mathcal{L}_{\gamma-\frac{1}{G}}^{F}(h)+G\cdot \left(8\widehat{R}_{\widetilde{m}}(\mathcal{H}) + \frac{4+17\sqrt{\ln(4/\delta)}}{\sqrt{\widetilde{m}}} \right)
\end{equation*}
\begin{comment}
%% in case Guy wants the full version
\begin{align*}
\widehat{\mathcal{L}}_{\gamma}^{F}(h) & \triangleq\widehat{\mathcal{L}}_{\gamma}^{\sigma}(h)\\
 & \leq\widehat{\mathcal{L}}_{\gamma-\frac{1}{G}}^{\tau^{G}}(h)\\
 & \leq\mathcal{L}_{\gamma-\frac{1}{G}}^{\tau^{G}}(h)+\left(8G\cdot R_{\widetilde{m}}(\mathcal{H})+\frac{4G}{\sqrt{\widetilde{m}}}+\sqrt{\frac{\ln(2/\delta)}{2\widetilde{m}}}\right)\\
 & \leq\mathcal{L}_{\gamma-\frac{1}{G}}^{\tau^{G}}(h)+\left(8G\widehat{R}_{\widetilde{m}}(\mathcal{H})+\frac{4G}{\sqrt{\widetilde{m}}}+3\sqrt{\frac{\ln(2/\delta)}{2\widetilde{m}}}\right)\\
 & \leq\mathcal{L}_{\gamma-\frac{1}{G}}^{\sigma}(h)+\left(8G\widehat{R}_{\widetilde{m}}(\mathcal{H})+\frac{4G}{\sqrt{\widetilde{m}}}+3\sqrt{\frac{\ln(2/\delta)}{2\widetilde{m}}}\right)\\
 & =\mathcal{L}_{\gamma-\frac{1}{G}}^{F}(h)+\left(8G\widehat{R}_{\widetilde{m}}(\mathcal{H})+\frac{4G}{\sqrt{\widetilde{m}}}+3\sqrt{\frac{\ln(2/\delta)}{2\widetilde{m}}}\right)
\end{align*}
\end{comment}

Combining both the previous inequalities and using the union bound,
we obtain that for every $G\geq0$, with probability at least $1-\delta$,
simultaneously $\forall h\in\mathcal{H}$
\[
\mathcal{L}_{\gamma+\frac{1}{G}}^{F}(h)-\Delta_{\widetilde{m}} \leq\widehat{\mathcal{L}}_{\gamma}^{F}(h)\leq\mathcal{L}_{\gamma-\frac{1}{G}}^{F}(h)+\Delta_{\widetilde{m}}
\]
where $ \Delta_{m}=2G\cdot\left(4\widehat{R}_{\frac{m-1}{2}}\left(\mathcal{H}\right)+\frac{4+17\sqrt{\ln(4/\delta)}}{\sqrt{m-1}}\right)$, as required.

\end{proofof}

\section{(Fair and Accurate) Learning Objectives}
\label{sec:accuracy}

\subsection{Preliminaries}

We will consider a learning problem as given by a 5-tuple $\left(\X,\left(\Y,\Y'\right),\mathcal{H},\left(\mathcal{L}^U, \L^F\right)\right)$. In this notation,  $\X$ is the set of instances; $\Y$ is the set of possible labels assigned to instances; $\Y'\supseteq \Y$ is the set of labels the learner is allowed to return; $\H$ is a fixed family of predictors $h:\X\rightarrow \Y'$ which we require the learner to compete with (in terms of the returning a classifier with loss comparable to the best loss in $\H$); $\L^U$ is a univariate (``utility'') loss function used to measure the discrepancy between the predicted outputs and the true labels; and finally, $\L^F$ is a bivariate  (``fairness'') loss function. We denote $\mathcal{L}_{\mathcal{D}}^{U}\left(\mathcal{H}\right)=\min_{h\in\mathcal{H}}\left(\mathcal{L}_{\mathcal{D}}^{U}(h)\right)$.

For the remainder of this work, we will focus on the setting of of binary classification, i.e $\Y=\left\{ \pm1\right\}$. To accommodate probabilistic classifiers, we will consider $\Y'=\left[0,1\right]$, such that we interpret a classifier $h:\,\X\rightarrow \Y'=\left[0,1\right]$ as predicting the label 1 for an instance $x$ with probability $h(x)$, and $-1$ otherwise.  For utility we will mostly use the absolute loss: $\ell_{abs}\left(h,(x,y)\right)=\left|h(x)-y\right|$. For fairness we use the approximate metric-fairness loss defined in Section~\ref{sec:approximate_fairness}.

\subsection{PAC Learnability under a fairness constraint}

From a learning perspective, we say that a learning algorithm is fair if w.h.p, it returns a sufficiently-fair classifier.

\begin{definition}[($\alpha,\gamma$)-fair learning algorithm]
A learning algorithm $A$ is ($\alpha,\gamma$)-fair for a given metric $d$ if there exists a function $m_{A}(\alpha, \gamma, \delta):(0,1)^{3}\rightarrow\mathbb{N}$ such that for any distribution $\D$ and for every $\delta \in(0,1)$, if $S$ is an  i.i.d sample from $\D$ of size $m\geq m_A(\alpha, \gamma, \delta)$, then w.p greater than $1-\delta$, the output of $A(S)$ is $(\alpha,\gamma)$-fair w.r.t $\D,d$.
\end{definition}

Note that any learning algorithm $A$ that completely disregards the sample and simply outputs a constant predictor (e.g, $h(x)\equiv1$) satisfies the condition in Equation~\ref{definition:perfect_fairness} and is therefore a (0,0)-fair learning algorithm. In other words, fair learning is, in itself, trivial. It is the combination of a fairness and an accuracy objective that makes for an interesting task. In this work, we focus on the objective of maximizing utility subject to a constraint on the fairness loss. This is a natural formulation, because we think of fairness as a hard (often externally imposed) requirement that cannot necessarily be traded off for better accuracy. The objective of finding the most accurate sufficiently-fair hypothesis is summarized in the following definition. Crucially, both fairness and accuracy goals are stated w.r.t the unknown underlying distribution.

\begin{definition}[PACF Learnability: Probably-Approximately Correct and Fair]
\label{def:pacf}
A hypothesis class $\mathcal{H}$ is PACF learnable with respect to a univariate “utility” loss function $\ell^{U}:\,\mathcal{H}\times Z\rightarrow\mathbb{R}_{+}$ and a bivariate “fairness” loss function $\ell^{F}_{d}:\mathcal{H}\times Z\times Z\rightarrow\mathbb{R}_{+}$ if there exists a function $m_{\mathcal{H}}(\alpha, \gamma, \delta, \epsilon, \epsilon_\alpha, \epsilon_\gamma):(0,1)^{6}\rightarrow\mathbb{N}$ (polynomial in $\alpha, \gamma, \epsilon, \epsilon_\alpha, \epsilon_\gamma,\log\frac{1}{\delta}$) and a learning algorithm with the following property: For every required fairness parameters $\alpha, \gamma \in(0,1)$, failure probability  $\delta\in(0,1)$ and error parameters $\epsilon, \epsilon_\alpha, \epsilon_\gamma\in(0,1)$ and for every distribution $\mathcal{D}$ over $Z=\X \times \Y$, when running the learning algorithm on $m\geq m_{\mathcal{H}}(\alpha, \gamma,  \delta, \epsilon, \epsilon_\alpha, \epsilon_\gamma)$ i.i.d. examples generated by $\mathcal{D}$, the algorithm returns a hypothesis $h$ such that, with probability of at least $1-\delta$ (over the choice of the $m$ training examples),
\begin{align}
 & \mathcal{L}_{\mathcal{D},\gamma,d}^{F}(h)\leq\alpha\\
 & \mathcal{L}_{\mathcal{D}}^{U}(h)\leq\mathcal{L}_{\mathcal{D}}^{U}\left(\mathcal{H}^{\alpha-\epsilon_\alpha, \gamma-\epsilon_\gamma}\right)+\epsilon
\end{align}

where $\L^{U}_{\mathcal{D}}(h)=\mathbb{E}_{z\sim\mathcal{D}}\left[\ell^{U}(h,z)\right]$
and $\L^{F}_{\mathcal{D}}(h)=\mathbb{E}_{z,z'\sim\mathcal{D}}\left[\ell^{F}_{\gamma}(h,z,z')\right]$.

\end{definition}

We consider the above definition as \emph{strong} PACF learnability, and also define a relaxed notion, in which the learner must still output an hypothesis that is $\alpha$ fair, but in terms of accuracy is only required to compete  with $\mathcal{H}^{\alpha'}$, for some $0 \leq \alpha' \leq \alpha$, that does not need to approach $\alpha$ as $m\rightarrow\infty$. We formalize this using a function $g:[0,1]^2\rightarrow[0,1]^2$ that captures the degradation in the accuracy guarantee.

\begin{definition}[$g(\cdot)$-relaxed PACF Learnability]
\label{def:weak_pacf}
A hypothesis class $\mathcal{H}$ is $g(\cdot)$-relaxed PACF learnable with respect to a function $g: [0,1]^2\rightarrow [0,1]^2$, a univariate “utility” loss function $\ell^{U}:\,\mathcal{H}\times Z\rightarrow\mathbb{R}_{+}$ and a bivariate “fairness” loss function $\ell^{F}:\mathcal{H}\times Z\times Z\rightarrow\mathbb{R}_{+}$ if there exists a function $m_{\mathcal{H}}(\alpha, \gamma, \delta, \epsilon, \epsilon_\alpha, \epsilon_\gamma):(0,1)^{6}\rightarrow\mathbb{N}$ (polynomial in $\alpha, \gamma, \epsilon,\epsilon_\alpha, \epsilon_\gamma,\log\frac{1}{\delta}$), and a learning algorithm with the following property: For every required fairness parameters $\alpha, \gamma \in(0,1)$, failure probability  $\delta\in(0,1)$ and accuracy parameters $\epsilon,\epsilon_\alpha, \epsilon_\gamma,\in(0,1)$ and for every distribution $\mathcal{D}$ over $Z=\X \times \Y$, when running the learning algorithm on $m\geq m_{\mathcal{H}}(\alpha, \gamma, \epsilon, \epsilon_\alpha, \epsilon_\gamma, \delta)$ i.i.d. examples generated by $\mathcal{D}$, the algorithm returns a hypothesis $h$ such that, with probability of at least $1-\delta$ (over the choice of the $m$ training examples),
\begin{align}
 & \mathcal{L}_{\mathcal{D},\gamma,d}^{F}(h)\leq\alpha\\
 & \mathcal{L}_{\mathcal{D}}^{U}(h)\leq\mathcal{L}_{\mathcal{D}}^{U}\left(\mathcal{H}^{g(\alpha,\gamma)-(\epsilon_\alpha, \epsilon_\gamma)}\right)+\epsilon
\end{align}
\end{definition}

\section{Information Theoretic PACF Learnability}
Using the generalization result from Theorem~\ref{thm:rad_generalization}, we can now derive the sample complexity for \emph{strong} PACF learnability in the information theoretic setting.

\begin{theorem}[Information theoretic PACF learnability]
\label{thm:inf_fpac}
Suppose $\H$ is PAC learnable with sample complexity $m_{PAC}(\epsilon,\delta)$. Then it is PACF learnable with sample complexity
\[
m\left(\alpha,\gamma,\epsilon,\epsilon_{\alpha},\epsilon_{\gamma},\delta\right)=\max\left\{ m_{PAC}(\epsilon,\delta),\,\,\left(\frac{8+34\sqrt{\ln(4/\delta)}}{\epsilon_{\alpha}\epsilon_{\gamma}-8\widehat{R}_{\frac{m-1}{2}}(\mathcal{H})}\right)^2+1\right\}
\]
\end{theorem}

\begin{proofof}{Theorem~\ref{thm:inf_fpac}}

Let $\alpha,\gamma\in(0,1)$ required fairness parameters, $\delta\in(0,1)$
required failure probability and $\epsilon,\epsilon_{\alpha},\epsilon_{\gamma}\in(0,1)$
error parameters. Let $m,G$ be parameters to be determined later.
Set $\widetilde{\gamma}=\gamma-\frac{1}{G}$ and $\widetilde{\alpha}=\alpha-\Delta_{m}$,
for $\Delta_{m}=2G\cdot\left(4\widehat{R}_{\frac{m-1}{2}}\left(\mathcal{H}\right)+\frac{4+17\sqrt{\ln(4/\delta)}}{\sqrt{m-1}}\right)$.

Let $h^{\star}$ denote the solution to the fairness-constrained ERM,

\[
h^{\star}=\arg\min_{h\in\mathcal{H}}\mathcal{\widehat{L}}^{U}(h)\,\,\text{subject to \ensuremath{h\in\widehat{H}^{\widetilde{\alpha},\widetilde{\gamma}}}}
\]

Let $m_{F}=\left(\frac{8+34\sqrt{\ln(4/\delta)}}{\epsilon_{\alpha}\epsilon_{\gamma}-8\widehat{R}_{\frac{m-1}{2}}(\mathcal{H})}\right)^2+1$ and $m_{PAC}=m_{PAC}(\epsilon,\delta)$.,

For fairness, we use Theorem~\ref{thm:rad_generalization} to prove that for every $m,G$, w.p at least $1-\frac{\delta}{2}$ over the choice of
$S\sim\mathcal{D}^{m}$, $h^{\star}$ is $\left(\alpha,\gamma\right)$-fair
w.r.t $\mathcal{D}$:

\begin{align*}
\mathcal{L}_{\gamma}^{F}(h^{\star}) & \leq\mathcal{\widehat{L}}_{\gamma-\frac{1}{G}}^{F}(h^{\star})+\Delta_{m}\\
 & =\mathcal{\widehat{L}}_{\widetilde{\gamma}}^{F}(h^{\star})+\Delta_{m}\\
\left(h^{\star}\in\widehat{H}^{\widetilde{\alpha},\widetilde{\gamma}}\right) & \leq\widetilde{\alpha}+\Delta_{\widetilde{m}}\\
 & =\alpha
\end{align*}

For accuracy, we note that the known equivalence of learnability and uniform convergence in the regression setting \cite{alon1997scale} implies that for $m\geq m_{PAC}$, w.p at least $1-\delta/2$, for every $h\in \H$, $\mathcal{L}^{U}(h)\leq\mathcal{\widehat{L}}^{U}(h)+\epsilon$. Using this we obtain that w.p at least $1-\delta /2$, for $m\geq\max\left\{ m_{PAC},m_{PACF}\right\} $:

\begin{align*}
\mathcal{L}^{U}(h^{\star}) & \leq \mathcal{\widehat{L}}^{U}(h^{\star})+\epsilon \\
 & =\mathcal{\widehat{L}}^{U}(\widehat{H}^{\widetilde{\alpha},\widetilde{\gamma}})+\epsilon\\
 & =\mathcal{\widehat{L}}^{U}(\widehat{H}^{\alpha-\Delta_{\widetilde{m}},\gamma-\frac{1}{G}})+\epsilon\\
 & \leq\mathcal{\widehat{L}}^{U}(\widehat{H}^{\alpha-\epsilon_{\alpha},\gamma-\epsilon_{\gamma}})+\epsilon
\end{align*}

We now use the union bound to conclude that setting $m\geq\max\left\{ m_{PAC},m_{F}\right\} $
yields that with probability at least $1-\delta$ over the choice
of $S\sim\mathcal{D}^{m}$, $h$ satisfies all the conditions in Definition~\ref{def:pacf}, and is therefore (exact) PACF learnable.

\end{proofof}

\begin{comment}
\begin{theorem}[Information theoretic f-PAC learnability]
\label{thm:inf_fpac}
If $\mathcal{H}$ is finite then it is (information theoretically) f-PAC learnable with sample complexity $m\geq\max\left\{ \frac{\log\left(\frac{4\left|H\right|}{\delta}\right)+\epsilon_{1}^{2}}{\epsilon_{1}^{2}},\,\,\,\frac{\log\left(\frac{4\left|H\right|}{\delta}\right)}{\epsilon_{2}^{2}}\right\}$.
\end{theorem}
\end{comment}

\section{Efficient relaxed-PACF Learnability of Linear Predictors}
\label{sec:learning-linear}

In this section, we present our main result, which is that the classes of (linear and logistic) predictors\footnote{Since we interpret the output of the linear and logistic regression models are probabilities for a classification model, we refer to them as linear and logistic \emph{classifiers}, rather than regressors.} are efficiently $g(\cdot)$-relaxed PACF-learnable, with $g(\alpha,\gamma)=(\alpha\cdot\gamma - \gamma^\star, \gamma^\star)$, for every $\gamma^\star\in(0,1)$. We begin with the proof for the class of linear classifiers, which demonstrates the idea of a relaxation that allows for efficient solving of the optimization problem associated with the relaxed PACF learnability requirement. We then proceed to the case of logistic predictors, which is more involved due to the non-convexity of the fairness objective induced by a logistic transfer function. We overcome this difficulty using improper learning.

\subsection{Setting}

For the remainder of this section, we consider the problem of efficiently PACF learning the problem $\allowbreak \left(\mathcal{X},\left(\mathcal{Y},\mathcal{Y}'\right),H,\left(\mathcal{L}^{U},\mathcal{L}^{F}\right)\right)$, where: $\mathcal{X}$ is a compact subset of an RKHS, which w.l.o.g. will be taken to be the unit ball around the origin; $\mathcal{Y=}\left\{ \pm1\right\}$  and $\mathcal{Y}'=\left[0,1\right]$, since we're interested in performing classification using probabilistic classifiers; the utility loss $\mathcal{L}^{U}$ is the absolute value loss function and $\mathcal{L}^{F}$ is the approximate metric-fairness loss w.r.t some known metric $d$. We define hypothesis classes $H_{lin}$ and $H_{\phi}$, corresponding to linear and logistic regression predictors, as follows. $H_{lin}$ is the class of linear predictors, \footnote{The norm bound on $\mathbf{w}$ is so to ensure that $h(x)\in\left[-1,1\right]=\mathcal{Y}$}

\begin{equation}
H_{\lin} \overset {{\scriptscriptstyle {\scriptstyle \text{def}}}}{=}\left\{ \mathbf{x}\mapsto \frac{1 + \left\langle \mathbf{w},\mathbf{x}\right\rangle}{2} :\,\,|\mathbf{w}\|\leq1\right\},
\end{equation}

And $H_{\phi,L}$ is the class of logistic predictors, formed by composing a linear function with a sigmoidal transfer function:

\begin{equation}
	H_{\phi,L}\overset{{\scriptscriptstyle {\scriptstyle \text{def}}}}{=}\left\{ \mathbf{x}\mapsto\phi_{\ell}\left(\left\langle w,\mathbf{x}\right\rangle \right):\,\,\mathbf{w}\in\mathcal{X}, \ell \in [0,L]\right\}
\end{equation}

where $\phi_{\ell}:\,\left[-1,1\right]\rightarrow\left[0,1\right]$ is a continuous and $\ell$-Lipschitz sigmoid function, defined as $\phi_{\ell}(z)=\frac{1}{1+\exp(-4\ell \cdot z)}$.

\subsection{Linear Regression}

\begin{theorem}
\label{thm:learning_linear}
For every $\gamma^{\star}\in(0,1)$, $H_{lin}$ is relaxed PACF learnable with $g(\alpha,\gamma)=(\alpha\cdot\gamma-\gamma^{\star},\gamma^{\star})$ and sample and time complexities of $poly(\frac{1}{\epsilon_\gamma}, \frac{1}{\epsilon_\alpha}, \frac{1}{\epsilon}, \log \frac{1}{\delta})$

\end{theorem}

\begin{proofof}{Theorem~\ref{thm:learning_linear}}

The proof is structured as follows. First, we discuss why the immediate optimization problem associated with PACF learning $H_{lin}$ is not efficiently solvable. We then show a convex problem whose solution can be shown to meet the fairness and competitiveness requirements in Definition~\ref{def:weak_pacf}, but with respect to the sample. We conclude the proof by proving generalization of both the fairness and the competitiveness requirements.

For simplicity, we use $H=H_{lin}$ throughout the proof. Fix a sample $S$. The straight-forward approach to PACF learn $\allowbreak \left(\mathcal{X},\left(\mathcal{Y},\mathcal{Y}'\right),H,\left(\mathcal{L}^{U},\mathcal{L}^{F}\right)\right)$
would be to search $H$ for an hypothesis that minimizes the empirical
risk, subject to being sufficiently fair w.r.t the pairs from $M(S)$ (a random matching induced by the sample $S$). This is equivalent to solving the following optimization problem, for some appropriate setting of $\alpha,\gamma\in(0,1)$:

\begin{equation}
\label{eq:naive_pacf_h_lin}
\min_{\|\mathbf{w}\|\leq1}\sum_{i=1}^{m}\left|\left\langle \mathbf{w},\mathbf{x}_{i}\right\rangle -y_{i}\right|\,\,\text{subject to \,\,}\mathbf{w}\in\widehat{H}_{\ell_0}^{\alpha,\gamma}
\end{equation}

where $\widehat{H}_{\ell_0}^{\alpha,\gamma}$ denotes the set of hypotheses in $H$ that are $(\alpha,\gamma)$-fair w.r.t $M(S)$. We use the $\ell_0$ notation because:

\begin{equation}
\label{def:H_l0}
\mathbf{w}\in\widehat{H}_{\ell_{0}}^{\alpha,\gamma}\iff\begin{cases}
\left|\left\langle \mathbf{w},\mathbf{x}\right\rangle -\left\langle \mathbf{w},\mathbf{x}'\right\rangle \right|\leq d(\mathbf{x},\mathbf{x}')+\gamma+\xi_{e}(\mathbf{w}) & \forall e=\left(\mathbf{x},\mathbf{x}'\right)\in M(S)\\
\xi_{e}(\mathbf{w})\in\left\{ 0,1\right\}  & \forall e\in M(S)\\
\|\xi(\mathbf{w})\|_{0}\leq\alpha\cdot \left|M(S)\right|
\end{cases}
\end{equation}

here, $\xi_{e}(\mathbf{w})\in{0,1}$ is an indicator for whether the metric-fairness
constraint on the sample pair $e=\left(\mathbf{x},\mathbf{x}'\right)$
is violated by more than an additive $\gamma$ term or not, $\xi(\mathbf{w})\in\left\{ 0,1\right\} ^{\left|M(S)\right|}$ is the vector of violations on $M(S)$ induced by the linear classifier parametrized by $\mathbf{w}$, and $\|\xi(\mathbf{w})\|_{0}\leq\alpha\cdot \left|M(S)\right|$
ensures that the overall fraction of such violations does not exceed $\alpha$. We formalize the definition of a fairness-violation vector as follows:

\begin{definition} For every $\mathbf{w}\in\X$ and $\gamma\in(0,1)$, the induced violation vector $\xi^\gamma(\mathbf{w})\in\left\{ 0,1\right\} ^{\left|M(S)\right|}$ is defined as
\begin{equation*}
\left[\xi_{\gamma}(\mathbf{w})\right]_{i}=d(\mathbf{x},\mathbf{x}')+\gamma-\left|\left\langle \mathbf{w},\mathbf{x}\right\rangle -\left\langle \mathbf{w},\mathbf{x}'\right\rangle \right|
\end{equation*}
\end{definition}

Unfortunately, the constraint $w\in\widehat{H}_{\ell_{0}}^{\alpha,\gamma}$
is not convex (since $\widehat{H}_{\ell_{0}}^{\alpha,\gamma}$ is
not a convex set), and hence the optimization problem associated with
PACF learning $H$ (Equation~\ref{eq:naive_pacf_h_lin}) is not convex. We overcome this by replacing the $\ell_0$ constraint on the norm of $\xi^{\gamma}(\mathbf{w})$ with an $\ell_1$ constraint on the norm of $\xi^{0}(\mathbf{w})$ (note that we took here $\gamma=0$). We denote this $\ell_1$-norm-constrained version of $\widehat{H}^{\alpha,\gamma}$ as $\widehat{H}_{\ell_{1}}^{\alpha}$ (see Definition \ref{def:l1_fairness_loss} and the discussion that follows it):

\[
\mathbf{w}\in\widehat{H}_{\ell_{1}}^{\alpha}\iff\begin{cases}
\left|\left\langle \mathbf{w},\mathbf{x}\right\rangle -\left\langle \mathbf{w},\mathbf{x}'\right\rangle \right|\leq d(\mathbf{x},\mathbf{x}')+\xi_{e}(\mathbf{w}) & \forall e=\left(\mathbf{x},\mathbf{x}'\right)\in M(S)\\
0\leq\xi_{e}(\mathbf{w})\leq1 & \forall e\in M(S)\\
\|\xi(\mathbf{w})\|_{1}\leq\alpha
\end{cases}
\]

\begin{lemma}
\label{lemma:l1_convex}
 For every $\alpha\in(0,1)$
\begin{equation}
\label{eq:l1_convex}
\min_{\|\mathbf{w}\|\leq1}\sum_{i=1}^{m}\left|\left\langle \mathbf{w},\mathbf{x}_{i}\right\rangle -y_{i}\right|\,\,\text{subject to \,\,}\mathbf{w}\in\widehat{H}_{\ell_{1}}^{\alpha}
\end{equation}

is a convex optimization problem.
\end{lemma}
\begin{proofof}{Lemma~\ref{lemma:l1_convex}}
It's sufficient to prove that $\widehat{H}_{\ell_{1}}^{\alpha}$
is a convex set. Let $\mathbf{w}_{1},\mathbf{w}_{2}\in\widehat{H}_{\ell_{1}}^{\alpha}$,
we wish to show that $\forall t\in\left[0,1\right]$, $\mathbf{w}=t\mathbf{w}_{1}+(1-t)\mathbf{w}_{2}$
is also in $\widehat{H}_{\ell_{1}}^{\alpha}$.

Let $e=\left(\mathbf{x},\mathbf{x}'\right)\in M(S)$. Then,
\begin{align*}\left|\left\langle \mathbf{w},\mathbf{x}\right\rangle -\left\langle \mathbf{w},\mathbf{x}'\right\rangle \right| & =\left|\left\langle t\mathbf{w}_{1}+(1-t)\mathbf{w}_{2},\mathbf{x\mathrm{-}x'}\right\rangle \right|\\
 & \leq t\left|\left\langle \mathbf{w}_{1},\mathbf{x\mathrm{-}x'}\right\rangle \right|+(1-t)\left|\left\langle \mathbf{w}_{2},\mathbf{x\mathrm{-}x'}\right\rangle \right|\\
\left(\mathbf{w}_{1},\mathbf{w}_{2}\in\widehat{H}_{\ell_{1}}^{\alpha}\right) & \leq t\cdot\left(d(\mathbf{x},\mathbf{x}')+\xi_{e}(\mathbf{w}_{1})\right)+(1-t)\cdot\left(d(\mathbf{x},\mathbf{x}')+\xi_{e}(\mathbf{w}_{2})\right)\\
 &=d(\mathbf{x},\mathbf{x}') + t\cdot\xi_{e}(\mathbf{w}_{1})+(1-t)\cdot\xi_{e}(\mathbf{w}_{2})
\end{align*}

this implies that for every $e\in M(S)$, $\xi_{e}(\mathbf{w})=t\cdot\xi_{e}(\mathbf{w}_{1})+(1-t)\cdot\xi_{e}(\mathbf{w}_{2})$. Observe that since $t\in\left[0,1\right]$ and $0\leq\xi_{e}(\mathbf{w}_{1}),\xi_{e}(\mathbf{w}_{2})\leq1$, we also have that $0\leq\xi_{e}(\mathbf{w})\leq1$. Finally, the third constraint also holds, since we have that:

\begin{align*}
\|\xi(\mathbf{w})\|_{1} & =\sum_{e\in M(S)}\xi_{e}(\mathbf{w})\\
 & =\sum_{e\in M(S)}\left[t\cdot\xi_{e}(\mathbf{w}_{1})+(1-t)\cdot\xi_{e}(\mathbf{w}_{2})\right]\\
 & =t\cdot\sum_{e\in M(S)}\xi_{e}(\mathbf{w}_{1})+(1-t)\cdot\sum_{e\in M(S)}\xi_{e}(\mathbf{w}_{2})\\
\left(\mathbf{w}_{1},\mathbf{w}_{2}\in\widehat{H}_{\ell_{1}}^{\alpha}\right) & \leq t\cdot\alpha\cdot\left|M(S)\right|+(1-t)\cdot\alpha\cdot\left|M(S)\right|\\
 & =\alpha\cdot\left|M(S)\right|
\end{align*}

Thus, we conclude that $\mathbf{w}\in\widehat{H}_{\ell_{1}}^{\alpha}$.

\end{proofof}

At this point, we have an efficient algorithm (the program from Equation~\ref{eq:l1_convex}) that can produce an hypothesis which is empirically $\ell_1$-MF and is also competitive with the class of all such empirically $\ell_1$-MF hypotheses. Lemma \ref{lemma:l1_l0} implies that this algorithm can be used to produce an hypothesis that is sufficiently \emph{approximate metric-fair} and is also competitive in accuracy with a certain (smaller) class of approximately metric-fair classifiers (where in both cases, fairness is calculated w.r.t the sample). 

To arrive at satisfying the conditions of relaxed metric-fair learning, it's left to prove that these guarantees can be extended to also hold for the underlying distribution $\D$. To this end, we employ generalization arguments for both the utility and fairness loss. For utility (Claim~\ref{claim:gen_lin_H_utility}), we use a standard Rademacher-based uniform convergence result, stated for the general setting of a class of linear predictors with bounded norm in a RHKS; for fairness (Claim ~\ref{claim:h_remove_hat}) we use the
generalization result from Corollary~\ref{cor:linear_generalization}.

\begin{claim}
\label{claim:gen_lin_H_utility}
Let $H_{B,M}$ denote the class of linear predictors with norm bounded by $B$ in a RHKS with a reproducing kernel $K: \X \times \X \rightarrow \mathbb{R}$ that satisfies $M=\sup K(\mathbf{x},\mathbf{x}')$. For every $\epsilon,\delta\in(0,1)$, setting $m\geq\frac{BM}{\epsilon^{2}}\left(2+9\sqrt{\ln(8/\delta)}\right)$ yields that w.p at least $1-\delta$ over an i.i.d sample $S\sim \D^m$, every $h\in H_{B,M}$ satisfies $\widehat{\mathcal{L}}^{U}(h)\leq\mathcal{L}^{U}(h)+\epsilon$.
\end{claim}

\begin{claim}
\label{claim:h_remove_hat}
For every $\alpha,\gamma\in(0,1)$, $G\geq1$ and $\delta\in(0,1)$, with probability at least $1-\frac{\delta}{2}$ over an i.i.d sample $S\sim \D^m$,
\[
H_{\ell_{0}}^{\alpha-\rho,\gamma-\frac{1}{G}}\subseteq\widehat{H}_{\ell_{0}}^{\alpha,\gamma}\subseteq H_{\ell_{0}}^{\alpha+\rho,\gamma+\frac{1}{G}}
\]
 where $\rho=2G\cdot\frac{4+4\sqrt{2}+17\sqrt{\ln(4/\delta)}}{\sqrt{m-1}}$.
\end{claim}

\begin{proofof}{Claim~\ref{claim:h_remove_hat}}
Let $G\geq1$ and $\gamma\in(0,1)$. From Corollary~\ref{cor:linear_generalization} (recall that in our setting, $C=M=1$),
we obtain that w.p at least $1-\frac{\delta}{2}$ over an i.i.d sample $S\sim \D^m$,
\[
\widehat{\mathcal{L}}_{\gamma+\frac{2}{G}}^{F}(\mathbf{w})-\rho\leq\mathcal{L}_{\gamma+\frac{1}{G}}^{F}(\mathbf{w})\leq\widehat{\mathcal{L}}_{\gamma}^{F}(\mathbf{w})+\rho
\]

This implies that with all but $\frac{\delta}{2}$ probability, for every $\alpha\in(0,1)$, the following holds:
\begin{align*}
 & \mathbf{w}\in\widehat{H}_{\ell_{0}}^{\alpha,\gamma}\Rightarrow\widehat{\mathcal{L}}_{\gamma}^{F}(\mathbf{w})\leq\alpha\Rightarrow\mathcal{L}_{\gamma+\frac{1}{G}}^{F}(\mathbf{w})\leq\alpha+\rho\Rightarrow\mathbf{w}\in H_{\ell_{0}}^{\alpha+\rho,\gamma+\frac{1}{G}}\\
 & \mathbf{w}\in H_{\ell_{0}}^{\alpha-\rho,\gamma-\frac{1}{G}}\Rightarrow\mathcal{L}_{\gamma-\frac{1}{G}}^{F}(\mathbf{w})\leq\alpha-\rho\Rightarrow\widehat{\mathcal{L}}_{\gamma}^{F}(\mathbf{w})\leq\alpha\Rightarrow\mathbf{w}\in\widehat{H}_{\ell_{0}}^{\alpha,\gamma}
\end{align*}
which implies $H_{\ell_{0}}^{\alpha-\rho,\gamma-\frac{1}{G}}\subseteq\widehat{H}_{\ell_{0}}^{\alpha,\gamma}\subseteq H_{\ell_{0}}^{\alpha+\rho,\gamma+\frac{1}{G}}$,
as required.
\end{proofof}

We are now prepared to prove relaxed PACF learnability of $H$. Let
$\alpha,\gamma\in(0,1)$ be the required fairness parameters, $\delta\in(0,1)$
required failure probability and $\epsilon,\epsilon_{\alpha},\epsilon_{\gamma}\in(0,1)$
error parameters. Let $G, m$ be parameters to be determined later. Define
$ \widetilde{\alpha}=\left(\alpha-\rho\right)\cdot\widetilde{\gamma}$
and $\widetilde{\gamma}=\gamma-\frac{1}{G}$, and denote with $\mathbf{w}$ the solution to the program in Equation~\ref{eq:l1_convex} using the parameter $\widetilde{\alpha}$.

\begin{proposition}
\label{prop:H_fairness}
with probability at least $\ensuremath{1-\tfrac{\delta}{2}}$, $\mathbf{w}$ is $\left(\alpha,\gamma\right)$-fair
w.r.t $\mathcal{D}$.
\end{proposition}
\begin{proofof}{Proposition~\ref{prop:H_fairness}}

\begin{align*}
\mathcal{L}_{\gamma}^{F}(\mathbf{w}) & \leq\mathcal{L}^{F}\left(\widehat{H}_{\ell_{1}}^{\widetilde{\alpha}}\right)\\
\left(\text{Lemma ~\ref{lemma:l1_l0}, with \ensuremath{\widetilde{\gamma}}}\right) & \leq\mathcal{L}_{\gamma}^{F}\left(\widehat{H}_{\ell_{0}}^{\frac{\widetilde{\alpha}}{\ensuremath{\widetilde{\gamma}}},\ensuremath{\widetilde{\gamma}}}\right)\\
\left(\text{Claim~\ref{claim:h_remove_hat}, w.p \ensuremath{1-\tfrac{\delta}{2}}}\right) & \leq\mathcal{L}_{\gamma}^{F}\left(H_{\ell_{0}}^{\frac{\widetilde{\alpha}}{\ensuremath{\widetilde{\gamma}}}+\rho,\ensuremath{\widetilde{\gamma}}+\frac{1}{G}}\right)\\
 & =\mathcal{L}_{\gamma}^{F}\left(H_{\ell_{0}}^{\alpha,\gamma}\right) \\
 & = \alpha
\end{align*}
\end{proofof}

\begin{proposition}
\label{prop:lin_accuracy}
Setting \[
m\geq\max\left\{ \left(\frac{\sqrt{2}+\sqrt{\ln(8/\delta)}}{\sqrt{2}\epsilon}\right)^{2},\,\,\left(\frac{4\left(4+4\sqrt{2}+17\sqrt{\ln(4/\delta)}\right)}{(1-\alpha)\cdot\epsilon_{\alpha}\cdot\min\left\{ \epsilon_{\alpha},\frac{\epsilon_{\gamma}}{2}\right\} }\right)^{2}\right\}
\] ensures that with probability at least $1-\tfrac{\delta}{2}$ over an i.i.d sample $S\sim \D^m$, for every $\gamma^{\star}$,
\begin{equation*}
\mathcal{L}^{U}(\mathbf{w})\leq\mathcal{L}^{U}\left(H^{\alpha\gamma-\gamma^{\star}-\epsilon_{\alpha},\gamma^{\star}-\epsilon_{\gamma}}\right)+\epsilon
\end{equation*}
\end{proposition}
\begin{proofof}{Proposition~\ref{prop:lin_accuracy}}

First, let $m^{1}=\left(\frac{2\cdot\left(\sqrt{2}+\sqrt{\ln(8/\delta)}\right)}{\sqrt{2}\epsilon}\right)^{2}$.
\begin{align*}
\mathcal{L}^{U}(\mathbf{w}) & =\\
\left(\text{Claim~\ref{claim:gen_lin_H_utility}, w.p \ensuremath{1-\tfrac{\delta}{8}} and $m\geq m^1$}\right) & \leq\mathcal{\widehat{L}}^{U}(\mathbf{w})+\tfrac{\epsilon}{2}\\
 & =\mathcal{\widehat{L}}^{U}(\widehat{H}_{\ell_{1}}^{\widetilde{\alpha}})+\tfrac{\epsilon}{2}\\
\left(\text{Lemma~\ref{lemma:l1_l0} with \ensuremath{\widetilde{\alpha}},\ensuremath{\gamma^{\star}}}\right) & \leq\mathcal{\widehat{L}}^{U}(\widehat{H}_{\ell_{0}}^{\widetilde{\alpha}-\ensuremath{\gamma^{\star}},\ensuremath{\gamma^{\star}}}) +\tfrac{\epsilon}{2} \\
\left(\text{Claim~\ref{claim:h_remove_hat}, w.p \ensuremath{1-\tfrac{\delta}{8}}}\right) & \leq\mathcal{\widehat{L}}^{U}(H_{\ell_{0}}^{\widetilde{\alpha}-\ensuremath{\gamma^{\star}}-\rho,\,\ensuremath{\gamma^{\star}}-\frac{1}{G}}) +\tfrac{\epsilon}{2} \\
\left(\text{Claim~\ref{claim:gen_lin_H_utility} w.p \ensuremath{1-\tfrac{\delta}{4}}, and $m\geq m^1$}\right) & \leq\mathcal{L}^{U}(H_{\ell_{0}}^{\widetilde{\alpha}-\ensuremath{\gamma^{\star}}-\rho,\,\ensuremath{\gamma^{\star}}-\frac{1}{G}})+\tfrac{\epsilon}{2} +\tfrac{\epsilon}{2} \\
& =\mathcal{L}^{U}(H_{\ell_{0}}^{\left(\alpha-\rho\right)\left(\gamma-\frac{1}{G}\right)-\ensuremath{\gamma^{\star}}-\rho,\ensuremath{\gamma^{\star}}-\frac{1}{G}})+\epsilon
\end{align*}

It is left to show that there is some setting of $G$ and a sufficiently large sample size $m$ for which it holds that
\[
\begin{cases}
\left(\alpha-\rho\right)\left(\gamma-\frac{1}{G}\right)-\ensuremath{\gamma^{\star}}-\rho\geq\alpha\gamma-\gamma^{\star}-\epsilon_{\alpha}\\
\ensuremath{\gamma^{\star}}-\frac{1}{G}\geq\gamma^{\star}-\epsilon_{\gamma}
\end{cases}
\]
Denote $\epsilon'=\min\left\{ \epsilon_{\alpha},\frac{\epsilon_{\gamma}}{2}\right\}$ and set $G=\frac{1}{\epsilon'}$.  The second equation is now satisfied, because $\frac{1}{G}=\epsilon'\leq\frac{\epsilon_{\gamma}}{2}\leq\epsilon_{\gamma}$. Plugging this into the first equation and simplifying, we obtain:
\[
(\alpha-\rho)(\gamma-\epsilon')-\gamma^{\star}-\rho\geq\alpha\gamma-\gamma^{\star}-\epsilon_{\alpha}	\iff\rho\leq\frac{\epsilon_{\alpha}-\alpha\cdot\epsilon'}{1+\gamma-\epsilon'}
\]
For simplicity, denote $\rho=\frac{G\overbrace{\left(8+8\sqrt{2}+34\sqrt{\ln(4/\delta)}\right)}^{z(\delta)}}{\sqrt{m-1}}=\frac{z(\delta)}{\epsilon'\cdot\sqrt{m-1}}$ and note that $\frac{\epsilon_{\alpha}-\alpha\cdot\epsilon'}{1+\gamma-\epsilon'}\geq\frac{(1-\alpha)\cdot\epsilon_{\alpha}}{2}$. It therefore suffices to choose $m$ such that:
\[
\frac{z(\delta)}{\epsilon'\cdot\sqrt{m-1}}\leq\frac{(1-\alpha)\cdot\epsilon_{\alpha}}{2}\iff m\geq\left(\frac{4\left(4+4\sqrt{2}+17\sqrt{\ln(4/\delta)}\right)}{(1-\alpha)\cdot\epsilon_{\alpha}\cdot\min\left\{ \epsilon_{\alpha},\frac{\epsilon_{\gamma}}{2}\right\} }\right)^{2}+1
\]

We can now use the union bound to conclude that for $m$ stated in the statement of Proposition~\ref{proposition:sigmoid_accuracy}, w.p at least $1-\frac{\delta}{2}$ over an i.i.d sample $S\sim \D^m$, it holds that
\[
\mathcal{L}^{U}(\mathbf{w})\leq\mathcal{L}^{U}(H_{\ell_{0}}^{\left(\alpha-\rho\right)\left(\gamma-\frac{1}{G}\right)-\ensuremath{\gamma^{\star}}-\rho,\ensuremath{\gamma^{\star}}-\frac{1}{G}})+\epsilon\leq\mathcal{L}^{U}(H_{\ell_{0}}^{\alpha\gamma-\gamma^{\star}-\epsilon_{\alpha},\gamma^{\star}-\epsilon_{\gamma}})+\epsilon
\]
as required.

\end{proofof}

To conclude the proof of Theorem~\ref{thm:learning_linear}, we note that by Propositions~\ref{prop:H_fairness} and~\ref{prop:lin_accuracy}, we have that w.p at least $1-\delta$ over an i.i.d sample $S\sim \D^m$ all the conditions in Definition~\ref{def:weak_pacf} are met for $g(\alpha,\gamma)=(\alpha\cdot\gamma-\gamma^{\star},\gamma^{\star})$ and $m$ as in the definition of Proposition~\ref{prop:lin_accuracy}. This implies that for this $g(\cdot)$, $H$ is $g(\cdot)$-relaxed PACF learnable with sample complexity $m$ and in time $poly(m)$, as required.

\end{proofof}

\subsection{Logistic Regression}

\begin{theorem}
\label{thm:log_regression_pacf}
For every constant $L > 0$ and for every $\gamma^\star\in(0,1)$, $H_{\sig,L}$  is relaxed PACF learnable with $g(\alpha,\gamma)=(\alpha\cdot\gamma-\gamma^{\star},\gamma^{\star})$, with sample and time complexities that are polynomial in $(\frac{1}{\epsilon_\gamma}, \frac{1}{\epsilon_\alpha}, \frac{1}{\epsilon}, \log \frac{1}{\delta})$.
\end{theorem}

\begin{proofof}{Theorem~\ref{thm:log_regression_pacf}}

The optimization problem that we used to efficiently-learn $H_{lin}$ in the proof of Theorem~\ref{thm:learning_linear} (Equation~\ref{eq:l1_convex}) is no longer convex in the case of $H_\phi$, due to the addition of the sigmoidal transfer function $\phi$:

\begin{equation*}
\begin{aligned}
& \underset{\mathbf{w},\xi(\mathbf{w})}{\text{minimize}}
& &  \sum_{i=1}^{m}\left|\left\langle \mathbf{w},\mathbf{x}_{i}\right\rangle -y_{i}\right| \\
& \text{subject to}
& & \left|\phi\left(\left\langle \mathbf{w},\mathbf{x}\right\rangle \right)-\phi\left(\left\langle \mathbf{w},\mathbf{x}'\right\rangle \right)\right|\leq d(\mathbf{x},\mathbf{x}')+\xi_{e}(\mathbf{w}), & \forall e=(\mathbf{x},\mathbf{x}')\in M(S) \\
&&& 0\leq\xi_{e}(\mathbf{w})\leq1\  & \forall e\in M(S) \\
&&& \|\xi(\mathbf{w})\|_{1}\leq\alpha
\end{aligned}
\end{equation*}

This implies that convex optimization techniques can't be used to solve the above formulation efficiently. We overcome this using improper learning. Namely, instead of attempting to directly PACF learn $H_{\phi}$, we PACF learn some other large class of hypotheses, where the fairness constraint \emph{is} convex. We proceed to define this class, previously introduced in ~\cite{shalev2011learning}. Define the following kernel function $K:\,\mathcal{X}\times\mathcal{X}\rightarrow\mathbb{R}$:
\[
K\left(\mathbf{x},\mathbf{x'}\right)\overset{{\scriptscriptstyle {\scriptstyle \text{def}}}}{=}\frac{1}{1-\frac{1}{2}\left\langle \mathbf{x},\mathbf{x}'\right\rangle }
\]
Since this is a positive definite kernel function, there exists some
mapping $\psi:\,\mathcal{X}\rightarrow\mathbb{V}$, where $\mathbb{V}$
is a RHKS with $\left\langle \psi(\boldsymbol{x}),\psi(\boldsymbol{x}')\right\rangle =K(\boldsymbol{x},\boldsymbol{x}')$.
Denote $\mathbb{V}_{B}=\left\{ \mathbf{v}\in\mathbb{V}\,\,\vert\,\,\|\mathbf{v}\|^{2}\leq B\right\} \subseteq\mathbb{V}$
and consider the class of linear classifiers parametrized by vectors
in $\mathbb{V}_{B}$:
\[
H_{B}\overset{{\scriptscriptstyle {\scriptstyle \text{def}}}}{=}\left\{ \mathbf{x}\mapsto\left\langle v,\psi(\mathbf{x})\right\rangle :\,\,\,\,\mathbf{v}\in\mathbb{V}_{B}\right\}
\]
We show that the optimization problem
associated with PACF learning $H_{B}$ can be solved efficiently,
and that its solution satisfies the requirements for
relaxed PACF learning of $H_{\phi}.$

\begin{claim}
$H_{B}$ is efficiently relaxed-PACF learnable.
\label{claim:solving_H_B}
\end{claim}

\begin{proofof}{Claim~\ref{claim:solving_H_B}}
PACF learning $H_{B}$ requires solving the following program:

\begin{equation*}
\begin{aligned}
& \underset{\xi(\mathbf{v}), \,\, \mathbf{v}\in\mathbb{V}_{B}}{\text{minimize}}
& & \sum_{i=1}^{m}\left|\left\langle \mathbf{v},\psi(\mathbf{x}_{i})\right\rangle -y_{i}\right| \\
& \text{subject to}
& & \left|\left\langle \mathbf{v},\psi(\mathbf{x})\right\rangle -\left\langle \mathbf{v},\psi(\mathbf{x}')\right\rangle \right|\leq d(\mathbf{x},\mathbf{x}')+\xi_{e}(\mathbf{v}), & \forall e=(\mathbf{x},\mathbf{x}')\in M(S) \\
&&& 0\leq\xi_{e}(\mathbf{v})\leq1\  & \forall e\in M(S) \\
&&& \|\xi(\mathbf{v})\|_{1}\leq\alpha
\end{aligned}
\end{equation*}

Or its equivalent re-parametrization\footnote{The two formulations are equivalent in the sense that any solution to the first program (for some desired $B$) can be obtained by solving Program 2 (for an appropriate value $\lambda_{B}$), see Theorem 1 in~\cite{oneto2016tikhonov}.}, which we refer to as Program 2:

\begin{equation*}
\begin{aligned}
& \underset{\mathbf{v}\in\mathbb{V}\,\,,\xi(\mathbf{v})}{\text{minimize}}
& & \sum_{i=1}^{m}\left|\left\langle \mathbf{v},\psi(\mathbf{x}_{i})\right\rangle -y_{i}\right|+\lambda_{B}\|\mathbf{v}\|_{2}^{2} \\
& \text{subject to}
& & \left|\left\langle \mathbf{v},\psi(\mathbf{x})\right\rangle -\left\langle \mathbf{v},\psi(\mathbf{x}')\right\rangle \right|\leq d(\mathbf{x},\mathbf{x}')+\xi_{e}(\mathbf{v}), & \forall e=(\mathbf{x},\mathbf{x}')\in M(S) \\
&&& 0\leq\xi_{e}(\mathbf{v})\leq1\  & \forall e\in M(S) \\
&&& \|\xi(\mathbf{v})\|_{1}\leq\alpha
\end{aligned}
\end{equation*}

Note that while the constraints are now linear, the mapping $\psi$ induced by the kernel $K$ is possibly infinite dimensional, and hence it is not obvious that this optimization problem can be solved efficiently. Fortunately, we can use The Representer Theorem~\cite{wahba1990spline} to reduce Program 2 to a finite dimensional optimization problem.

\begin{lemma}
\label{lemma:program2_rep_thm}
There exists a solution to Program 2 of the form $\mathbf{v}^{*}=\sum_{\ell=1}^{m}\beta_{\ell}\psi(\mathbf{x}_{\ell})$, where for every $\ell\in [m]$, $\beta_\ell \in [0,1]$.
\end{lemma}

\begin{proofof}{Lemma~\ref{lemma:program2_rep_thm}}

According to the Representer Theorem, any problem
of the form
\[
\min_{\mathbf{w}}\left\{ f\left(\left\langle \mathbf{w},\psi(\mathbf{x}_{1})\right\rangle ,\dots,\left\langle \mathbf{w},\psi(\mathbf{x}_{m})\right\rangle \right)+R\left(\|\mathbf{w}\|\right)\right\}
\]

for an arbitrary function $f\,:\,\mathbb{R}^{m}\rightarrow\mathbb{R}$
and a monotonically nondecreasing function $R\,:\,\mathbb{R}_{+}\rightarrow\mathbb{R}$,
has a solution of the form $\mathbf{v}^{*}=\sum_{i=1}^{m}\beta_{i}\psi(\mathbf{x}_{i})$.
Hence, it's sufficient to prove that the optimization problem in Program
2 is an instance of the problem $\min_{\mathbf{w}}\left\{ f\left(\left\langle \mathbf{w},\psi(\mathbf{x}_{1})\right\rangle ,\dots,\left\langle \mathbf{w},\psi(\mathbf{x}_{m})\right\rangle \right)+R\left(\|\mathbf{w}\|\right)\right\} $,
for some ``permissible'' choice of $f$ and $R$. We show this by taking $a_{i}=\left\langle \mathbf{w},\psi(\mathbf{x}_{i})\right\rangle $
and defining $R(a)=\lambda a^{2}$ and
\[
f(a_{1},\dots,a_{m})=\begin{cases}
\frac{1}{m}\sum_{i=1}^{m}\left|a_{i}-y_{i}\right| & \text{if }\exists\xi\in[0,1]^{\mathbb{R}^{|M(S)|}}\text{ s.t }\|\xi\|_{1}\leq\alpha \\ & \text{and }\forall e=(x_i,x_j)\in M(S),\,\,\,\left|a_{i}-a_{j}\right|\leq d(x_{i},x_{j})+\xi_{e}\\
\infty & \text{otherwise}
\end{cases}
\]
\end{proofof}

Lemma~\ref{lemma:program2_rep_thm} implies we can instead optimize over $\beta_{1}\dots\beta_{m}$. This yields the following optimization problem:

\begin{equation*}
\begin{aligned}
& \underset{\xi, \beta}{\text{minimize}}
& & \sum_{i=1}^{m}\left|\sum_{\ell=1}^{m}\beta_{\ell}K(\mathbf{x}_{\ell},\mathbf{x}_{i})-y_{i}\right|+\lambda_{B}\sum_{i,j=1}^{m}\beta_{i}\beta_{j}K(\mathbf{x}_{i},\mathbf{x}_{j}) \\
& \text{subject to}
& & \left|\sum_{\ell=1}^{m}\beta_{\ell}K(\mathbf{x}_{\ell},\mathbf{x})-\sum_{\ell=1}^{m}\beta_{\ell}K(\mathbf{x}_{\ell},\mathbf{x}')\right|\leq d(\mathbf{x},\mathbf{x}')+\xi_{e}(\beta), & \forall e=(\mathbf{x},\mathbf{x}')\in M(S) \\
&&& 0\leq\xi_{e}(\beta)\leq1\  & \forall e\in M(S) \\
&&& \|\xi(\beta)\|_{1}\leq\alpha
\end{aligned}
\end{equation*}

We can now conclude the proof of Claim~\ref{claim:solving_H_B}, since this is a convex optimization problem in $O(m)$ variables and therefore can be solved in time
$poly(m,\log B,\log\frac{1}{\alpha})$ using standard optimization tools.
\end{proofof}

For convenience, we hereby denote the solution to the above program when instantiated with parameters $B$ and $\alpha$ as $\mathbf{w}_{B,\alpha}$. We define a learning algorithm $A$ that given parameters $B^{\star},\alpha^{\star}$ returns $\mathbf{w}_{B^{\star},\alpha^{\star}}$.

Now, let $\alpha,\gamma\in(0,1)$ be the required fairness parameters,
$\delta\in(0,1)$ required failure probability, $\epsilon,\epsilon_{\alpha},\epsilon_{\gamma}\in(0,1)$
error parameters and $\gamma^\star\in(0,1)$. Recall (Definition~\ref{def:weak_pacf}) that in order to prove that $H_{\phi}$ is $g(\cdot)$-relaxed
PACF learnable for $g(\alpha,\gamma)=(\alpha\cdot\gamma-\gamma^\star,\gamma^\star)$,
we must prove that there is some setting of $B^{\star},\alpha^{\star}$
and a sample size $m\leq poly(\frac{1}{\alpha},\frac{1}{\gamma},\frac{1}{\epsilon},\frac{1}{\epsilon_{\alpha}},\frac{1}{\epsilon_{\gamma}},\log\frac{1}{\delta})$,
for which running $A$ on a sample $S\sim\mathcal{D}^{m}$ yields $\mathbf{w}=A(B^{\star},\alpha^{\star})$ that satisfies:

\[
\begin{cases}
\mathcal{L}_{\gamma}^{F}(\mathbf{w})\leq\alpha & \left(\text{fairness}\right)\\
\mathcal{L}^{U}(\mathbf{w})\leq\mathcal{L}^{U}\left(H_{\phi}^{\alpha\gamma-\gamma^\star -\epsilon_{\alpha},\gamma^\star-\epsilon_{\gamma}}\right) & \left(\text{accuracy}\right)
\end{cases}
\]

Let $G, B$ parameters to be defined later, and define:	
\begin{align*}
 & \rho=2G\cdot\frac{4+8\sqrt{B}+17\sqrt{\ln(4/\delta)}}{\sqrt{m-1}}\\
 & \widetilde{\gamma}=\gamma-\frac{1}{G}\\
 & \widetilde{\alpha}=\left(\alpha-\rho\right)\cdot\widetilde{\gamma}
\end{align*}

\begin{proposition}
\label{proposition:sigmoid_fairness}
For every $B>0$, w.p at least $1-\frac{\delta}{2}$ over a sample $S\sim\D^m$, $\mathbf{w}=A(B,\widetilde{\alpha})$ is $\left(\alpha,\gamma\right)$-fair w.r.t $\mathcal{D}$.
\end{proposition}
\begin{proofof}{Proposition~\ref{proposition:sigmoid_fairness}}
The argument follows directly from the same arguments in the proof of Proposition~\ref{prop:H_fairness} in Theorem~\ref{thm:learning_linear}, since $H_B$ is a class of linear classifiers.
\end{proofof}

\begin{proposition}
\label{proposition:sigmoid_accuracy}
Setting
\begin{align*}
 & \epsilon^\star=\min\left\{\epsilon, \epsilon_\alpha,\tfrac{\epsilon_\gamma}{2}\right\} \\
 & B^{\star}=6L^{4}+\exp\left(9L\log\left(\frac{4L}{\epsilon^\star}\right)+5\right) \\
 & m \geq m^\star =\max\left\{ \frac{2B^\star \cdot \left(2+9\sqrt{\ln(8/\delta)}\right)}{\epsilon^{2}},\,\,\,\left(\frac{4\left(4+8\sqrt{B^\star}+17\sqrt{\ln(4/\delta)}\right)}{(1-\alpha)\cdot\epsilon_{\alpha}\cdot \epsilon^\star} \right)^2+1\right\}
\end{align*}
\\
ensures that for every $\gamma^{\star}\in(0,1)$ and every $\delta, \alpha,\gamma, \epsilon_\alpha, \epsilon_\gamma, \epsilon \in(0,1)$, w.p at least $1-\frac{\delta}{2}$ over the choice of a sample $S\sim\D^m$, $\mathbf{w}=A(B^{\star},\alpha^{\star})$ satisfies
\[
\mathcal{L}^{U}(\mathbf{w})\leq\mathcal{L}^{U}\left(H_{\phi}^{\alpha\gamma-\gamma^{\star}-\epsilon_{\alpha},\gamma^{\star}-\epsilon_{\gamma}}\right) + \epsilon
\]

\end{proposition}
\begin{proofof}{Proposition~\ref{proposition:sigmoid_accuracy}}

The proof of Proposition~\ref{proposition:sigmoid_accuracy} will be based on the fact
that for sufficiently large $B$, $H_{\phi}^{\alpha,\gamma}$ is approximately contained (in terms of accuracy) in $H_{B}^{\alpha,\gamma}$. This is formalized in the following claim.

\begin{claim}
\label{claim:large_B}
For every $\alpha,\gamma,\epsilon\in(0,1)$ and $L\geq3$, setting $B=6L^{4}+\exp\left(9L\log\left(\frac{4L}{\epsilon}\right)+5\right)$ ensures that  $\L^{U}(H_{B}^{\alpha,\gamma+\epsilon})\leq \L^{U}(H_{\phi}^{\alpha,\gamma})+\frac{\epsilon}{2}$.
\end{claim}

\begin{proofof}{Claim~\ref{claim:large_B}}

From Lemma 2.5 in~\cite{shalev2011learning}, for every $\epsilon'\in(0,1)$, setting $B=6L^{4}+\exp\left(9L\log\left(\frac{2L}{\epsilon}\right)+5\right)$
yields that for any $h\in H_{\phi}$ there exists $h_{B}\in H_{B}$
such that
\[
\forall\mathbf{x}\in\mathcal{X},\,\,\,\left|h_{B}(\mathbf{x})-h(\mathbf{x})\right|\leq\epsilon'
\]
Let $\alpha,\gamma,\epsilon\in(0,1)$.
Now, by applying the above with $\epsilon'=\frac{1}{2}\epsilon$, we conclude that for every $h\in H_{\phi}^{\alpha,\gamma}$ there exists $h_{B}\in H_{B}$
such that $\forall\mathbf{x}\in\mathcal{X},\,\,\,\,\,\left|h_{B}(\mathbf{x})-h(\mathbf{x})\right|\leq\epsilon'$.
This implies that $\mathcal{L}^{U}(h_{B})\leq\mathcal{L}^{U}(h)+\epsilon'$.

To conclude the proof, it suffices to prove that $h_{B}\in H_{B}^{\alpha,\gamma+2\epsilon'}$.

Indeed:
\begin{align*}
\mathcal{L}_{\gamma+2\epsilon'}^{F}(h_{B}) & =\|\xi^{\gamma+2\epsilon'}(h_{B})\parallel_{0}\\
 & =\sum_{(x,x')\in M(S)}1\left[\left|h_{B}(x)-h_{B}(x')\right|>d(x,x')+\gamma+2\epsilon'\right]\\
 & =\sum_{(x,x')\in M(S)}1\left[\left|h_{B}(x)-h_{B}(x')\right|>d(x,x')+\gamma+2\epsilon'\right]\\
\star & \leq\sum_{(x,x')\in M(S)}1\left[\left|h(x)-h(x')\right|+\cancel{2\epsilon'}>d(x,x')+\gamma+\cancel{2\epsilon'}\right]\\
 & =\sum_{(x,x')\in M(S)}1\left[\left|h(x)-h(x')\right|>d(x,x')+\gamma\right]\\
 & =\|\xi^{\gamma}(h)\parallel_{0}\\
\left(h\in H_{\phi}^{\alpha,\gamma}\right) & \leq\alpha
\end{align*}

where in $\star$ we used the fact that from the triangle inequality,
\[
\forall\mathbf{x}\in\mathcal{X},\,\, \left|h_{B}(\mathbf{x})-h(\mathbf{x})\right|\leq\epsilon' \Longrightarrow \forall\mathbf{x},\mathbf{x'}\in\mathcal{X},\,\,\, \left|h_{B}(\mathbf{x})-h_{B}(\mathbf{x}')\right|\leq\left|h(\mathbf{x})-h(\mathbf{x}')\right|+2\epsilon'
\] \end{proofof} \\
We can now prove Proposition~\ref{proposition:sigmoid_accuracy}. Let $\gamma^\star\in(0,1)$ and $\mathbf{w}=A(B^\star, \alpha^\star)$, for $B^\star$ and $\alpha^\star$ as in the proposition's statement.

First, let $m^{1}=\frac{2B}{\epsilon^{2}}\left(2+9\sqrt{\ln(8/\delta)}\right)$ and $\epsilon'=\min\left\{ \epsilon,\epsilon_{\alpha},\frac{\epsilon_{\gamma}}{2}\right\}$. Now:
\begin{align*}
&  \mathcal{L}^{U}(\mathbf{w}) \\
\left(\text{Claim~\ref{claim:gen_lin_H_utility}, w.p \ensuremath{1-\tfrac{\delta}{4}} and \ensuremath{m\geq m^{1}} }\right) & \leq\mathcal{\widehat{L}}^{U}(\mathbf{w})+\frac{\epsilon}{4}\\
 & =\mathcal{\widehat{L}}^{U}\left(\widehat{H}_{B,\ell_{1}}^{\widetilde{\alpha}}\right)+\frac{\epsilon}{4}\\
\left(\text{Lemma~\ref{lemma:l1_l0} with \ensuremath{\widetilde{\alpha},\gamma^{\star}}}\right) & \leq\mathcal{\widehat{L}}^{U}\left(\widehat{H}_{B,\ell_{0}}^{\widetilde{\alpha}-\ensuremath{\gamma^{\star}},\gamma^{\star}}\right)+\frac{\epsilon}{4}\\
 & \triangleq\mathcal{\widehat{L}}^{U}\left(\widehat{H}_{B}^{\widetilde{\alpha}-\ensuremath{\gamma^{\star}},\gamma^{\star}}\right)+\frac{\epsilon}{4}\\
\left(\text{Claim~\ref{claim:h_remove_hat}, w.p \ensuremath{1-\tfrac{\delta}{2}}}\right) & \leq\mathcal{\widehat{L}}^{U}\left(H_{B}^{\widetilde{\alpha}-\ensuremath{\gamma^{\star}-\rho},\gamma^{\star}-\frac{1}{G}}\right)+\frac{\epsilon}{4}\\
\left(\text{Claim~\ref{claim:gen_lin_H_utility}, w.p \ensuremath{1-\tfrac{\delta}{4}} and \ensuremath{m\geq m^{1}} }\right) & \leq\mathcal{L}^{U}\left(H_{B}^{\widetilde{\alpha}-\ensuremath{\gamma^{\star}-\rho},\gamma^{\star}-\frac{1}{G}}\right)+\frac{\epsilon}{2}\\
\left(\text{Claim~\ref{claim:large_B} for \ensuremath{B\geq B^{\star}} and \ensuremath{\epsilon'} }\right) & \leq\mathcal{L}^{U}\left(H_{\phi}^{\widetilde{\alpha}-\ensuremath{\gamma^{\star}-\rho},\,\gamma^{\star}-\frac{1}{G}-\epsilon'}\right)+\frac{\epsilon}{2} + \frac{\epsilon'}{2}\\
 & \leq \mathcal{L}^{U}\left(H_{\phi}^{\left(\alpha-\rho\right)\cdot\left(\gamma-\frac{1}{G}\right)-\ensuremath{\gamma^{\star}}-\rho,\,\ensuremath{\gamma^{\star}}-\frac{1}{G}-\epsilon'}\right)+\epsilon
\end{align*}

It is left to show that there is some setting of $G$ and a sufficiently large sample size $m$ for which it also holds that
\[
\begin{cases}
\left(\alpha-\rho\right)\left(\gamma-\frac{1}{G}\right)-\ensuremath{\gamma^{\star}}-\rho\geq\alpha\gamma-\gamma^{\star}-\epsilon_{\alpha}\\
\ensuremath{\gamma^{\star}}-\frac{1}{G}-\epsilon' \geq\gamma^{\star}-\epsilon_{\gamma}
\end{cases}
\]

We set $G=\frac{1}{\epsilon'}$. The second inequality is satisfied, because $\frac{1}{G}=\epsilon'=2\epsilon'-\epsilon'\leq\epsilon_{\gamma}-\epsilon'$. Plugging this into the first equation and simplifying, we obtain:
\[
(\alpha-\rho)(\gamma-\epsilon')-\gamma^{\star}-\rho\geq\alpha\gamma-\gamma^{\star}-\epsilon_{\alpha}	\iff\rho\leq\frac{\epsilon_{\alpha}-\alpha\cdot\epsilon'}{1+\gamma-\epsilon'}
\]
which is satisfied for any $m\geq\left(\frac{4\left(4+8\sqrt{B}+17\sqrt{\ln(4/\delta)}\right)}{(1-\alpha)\cdot\epsilon_{\alpha}\cdot\min\left\{ \epsilon,\epsilon_{\alpha},\frac{\epsilon_{\gamma}}{2}\right\} }\right)^{2}+1$.

We can now use the union bound to conclude that for the sample size $m$ specified in the Proposition's statement, w.p at least $1-\frac{\delta}{2}$ over an i.i.d sample $S\sim \D^m$, it holds that
\[
\mathcal{L}^{U}(\mathbf{w})\leq\mathcal{L}^{U}(H_{\phi}^{\left(\alpha-\rho\right)\left(\gamma-\frac{1}{G}\right)-\ensuremath{\gamma^{\star}}-\rho,\ensuremath{\gamma^{\star}}-\frac{1}{G}-\epsilon})+\epsilon\leq\mathcal{L}^{U}(H_{\ell_{0}}^{\alpha\gamma-\gamma^{\star}-\epsilon_{\alpha},\gamma^{\star}-\epsilon_{\gamma}})+\epsilon
\]
as required.
\end{proofof} % proposition of accuracy

To conclude the proof of Theorem~\ref{thm:log_regression_pacf}, we note that by Propositions \ref{proposition:sigmoid_fairness} and \ref{proposition:sigmoid_accuracy}, we have that w.p at least $1-\delta$ over an i.i.d sample $S\sim \D^m$ all the conditions in Definition~\ref{def:weak_pacf} are met for $g(\alpha,\gamma)=(\alpha\cdot\gamma-\gamma^{\star},\gamma^{\star})$ and $m\geq m^\star$ as in the definition of Proposition \ref{proposition:sigmoid_accuracy}. This implies that for this $g(\cdot)$, $H_{\phi, L}$ is $g(\cdot)$-relaxed PACF learnable with sample complexity $m$ and in time $poly(m)$, as required. \end{proofof} % theorem of log. regression

\section{Intractability of Perfectly Metric-Fair Learning}
\label{sec:hardness}

We show that perfect metric-fairness can make simple learning tasks computationally intractable. Towards this, we exhibit a simple learning task that becomes intractable under a perfect metric-fairness constraint (for a particular metric). We note that this task {\em can} be solved in polynomial time under the approximate metric-fairness relaxation. See the discussion in Section \ref{sec:intro:hardness}.

\begin{theorem}[Intractability of perfectly metric-fair learning]
\label{thm:hardness}

Assume that one-way functions exist and let $\X$ be the unit ball in $\mathcal{R}^{n}$. There exist: $(i)$ a fixed distribution $\D$ over $(\X \times \pm 1)$, $(ii)$ a linear classifier $w: \X \rightarrow \pm 1$ that perfectly labels $\D$ ($\mathit{err}_{\D}(w)=0$), and $(iii)$ two efficiently sampleable distributions $U$ and $V$ on metrics $d: \X^2 \rightarrow [0,1]$, where:
\begin{enumerate}
\item \label{hardness:property_U} For {\em every} metric $d$ drawn from $U$, {\em every} perfectly metric-fair classifier $h:\X \rightarrow [0,1]$ (in any hypothesis class) has error $\mathit{err}_{\D}(h) = 1/2$.

\item \label{hardness:property_V} With overwhelming probability over a metric $d$ drawn from $V$, the (linear) classifier $w$ is perfectly metric-fair.

\item \label{hardness:indistinguishable} For every polynomial-time learning algorithm $\A$ there exists a constant $\alpha \in [0,1]$ such that one of the following two conditions holds:

\begin{enumerate}

\item Given a metric sampled from $U$, $\A$ outputs a classifier that violates perfect metric-fairness with probability almost $\alpha$.

\item Given a metric sampled from $V$, $\A$ outputs a classifier $h$ whose error is almost $1/2$ with probability at least $(1 - \alpha)$.

\end{enumerate}

\end{enumerate}

Moreover, the linear classifier $w$ not only labels examples in $\D$ correctly, it also has large margins (greater than $1/2$) on every example in $\D$'s support.

\end{theorem}

\begin{proof}[Proof of Theorem \ref{sec:intro:hardness}]

For the sake of readability, we choose to be somewhat informal in our treatment of asymptotics (i.e. we refer to a single distribution on learning problems and metrics, rather than an ensemble of distributions that grows with $n$). For an introduction to the foundations of cryptography, including pseudorandom generators, indistinguishability, and negligible quantities, we refer the reader to Goldreich \cite{Goldreich2001} .

We begin by specifying the distribution $\D$ and the classifier $w$. $\D$ will be the uniform distribution on the unit ball $\X$, conditioned on the $n$-th coordinate being either $1/2$ or $-1/2$. We restrict the last coordinate to ensure large margins (which are important for efficient PACF learnability). The linear classifier $w$ simply outputs the sign of the last coordinate, and each example $x \in \X$ gets the label $w(x)$. Note that indeed the linear classifier $w$ has perfect accuracy and large margins.

\paragraph{Metric distributions.} To describe the metric, we use a cryptographic pseudo-random generator (PRG) as follows. Recall that a PRG is a function such-that no polynomial-time algorithm can distinguish between a uniformly random string in $\zo^{2n}$, and the output of $G$ on a random string in $\zo^{n-1}$. Note that this is the case even though only a negligible fraction of the strings in $\zo^{2n}$ are in $G$'s image. A celebrated result of Hastad {\em et al.} \cite{HastadILL99} shows that PRGs can be constructed from any one-way function.

A metric $d$ in $U$ or in $V$ is described by a string $y \in \zo^{2n}$. The only difference between $U$ and $V$ will be in the distribution of $y$: in $U$ the vector $y$ will be pseudorandom (in the image of the generator), in $V$ the vector $y$ will be truly random (and with overwhelming probability not in the image of the generator). Given $y$, the distance between two distinct individuals $x,x' \in \R^n$ is determined as follows:

\begin{enumerate}

\item If $x$ and $x'$ get the same label, i.e. if $\sign(x[n]) = \sign(x'[n])$, then $d(x,x')=1$.

\item Otherwise, let $\Delta^{x,x'} \in \zo^{n-1}$ be computed as $\Delta^{x,x'}_i = \frac{1-(\sign(x_i) \cdot \sign(x'_i))}{2}$. I.e. $\Delta^{x,x'}_i$ is 0 if the signs of $x_i$ and $x'_i$ are identical, and 1 if they are different.

    If $G(\Delta^{x,x'})=y$, then $d(x,x')=0$.

\item Otherwise, $d(x,x')=1$.

\end{enumerate}

Note that for any choice of $y$, the above construction is indeed a pseduometric. The distance from $x$ to itself is defined to be 0, and distances are symmetric by symmetry of the $\oplus$ operation. Finally, for every $x,x',x'' \in X$ we have that: $$d(x,x') \leq d(x,x'') + d(x'',x').$$ To see this, observe that if $d(x,x') = 0$ then the LHS is bounded by the RHS. On the other hand, if $d(x,x') = 1$, then $x$ and $x'$ are distinct, and if $x''$ is also distinct from them, then it cannot be the case that both $d(x,x'')=0$ and $d(x'',x')=0$: one of these two distances must be 1.

We remark that in this construction we allow the distance between distinct points to be 0, and thus we get a distribution on pseudometrics. By defining the distance between examples $x$ and $x'$ s.t. $G(\Delta^{x,x'})=y$ to be a small positive quantity rather than 0 we would obtain a true metric, and the results are essentially unchanged.

Turning to prove the claimed properties, we have:

\paragraph{Property \ref{hardness:property_U}.} For metrics in $U$, $y$ is pseudorandom, where $G(s)=y$ for some $s \in \zo^{n-1}$. We can divide the examples in $\D$'s support into (disjoint) pairs $(x,x')$ s.t. $\Delta^{x,x'} = s$, where the label of $x$ is $1$ and the label of $x'$ is $-1$. Let $h$ be any perfectly metric-fair classifier. Since the $d(x,x')=0$, $h$ has to treat $x$ and $x'$ identically. Thus, the sum of $h$'s errors on these two examples must be exactly 1. Since we can partition the support of $\D$ into disjoint pairs of this form, we conclude that $h$'s overall error must be exactly $1/2$.

\paragraph{Property \ref{hardness:property_V}.} For metrics in $U$, $y$ is truly random. With overwhelming probability over the choice of $y$, there does not exist {\em any} $s \in \zo^{n-1}$ such that $G(s)=y$. When no such $s$ exists, the metric $d$ assigns distance 1 to {\em every} pair of distinct individuals. Thus, {\em every} hypothesis is perfectly metric-fair and in particular the classifier $w$ is a perfectly metric-fair classifier with error 0.

\paragraph{Property \ref{hardness:indistinguishable}.} Finally, since $G$ is a pseudorandom generator, no polynomial-time learner can distinguish between a metric (i.e. $y$) drawn from $U$ and a metric (i.e. $y$) drawn from $V$ (except with negligible advantage). For a learning algorithm $\A$, let $\alpha$ be the probability that $\A$, given a metric sampled from $V$, outputs a classifier whose error is noticeably less than $1/2$ (e.g. the error is no greater than $(1/2 - 1/n)$). Then, by the PRG's indistinguishability property, when given a metric sampled from $U$, the learner $\A$ must also output a classifier whose error is noticeably less than $1/2$  with probability almost $\alpha$ (e.g. at least $(\alpha-1/n)$). But by Property \ref{hardness:property_U}, whenever this is the case, $\A$ is also violating perfect metric-fairness. \end{proof}

\section{Conclusions and Future Directions}
\label{sec:conclusions}
%In this work we study \emph{metric-fair learning}, the task of learning (from finite samples) predictors that are fair w.r.t a given similarity metric. We have defined approximate metric fairness, a relaxation of individual-fairness requirement in \cite{DworkHPRZ12}, and leveraged the fact that approximate metric fairness \emph{generalizes} to construct polynomial-time learning algorithms that achieve competitive accuracy subject to fairness constraints for a rich class of linear predictors.

We conclude with several directions for future exploration:

\paragraph{The source of the similarity metric.} Throughout our work, we assumed that the similarity metric is known to the learner. This is a natural assumption in the scenario that the metric is used as a vehicle for knowledgeably correcting biases in the training data, or in domains where such metrics  naturally exist (such as credit scores and insurance risk scores). In other settings, however, relying on the existence of a metric is clearly a limitation (as also noted by \cite{DworkHPRZ12}). One potential avenue for future work is investigating the use of machine learning to recover a similarity metric from fairly labeled data.
%, in cases where the data is assumed to be generated ``fairly''.

\paragraph{Assumptions on the metric.} We make no assumptions about the similarity metric. In particular, it can be completely incompatible with accuracy and cryptographically contrived. Studying fairness-accuracy trade-offs imposed by particular similarity metrics
%between fairness and accuracy, as a function of the similarity metric and in general,
is an interesting direction for future research direction and could also be supplemented by empirical studies. Another interesting question is whether there exist natural classes of metrics for which the hardness results for perfect metric-fairness do not hold.

\paragraph{Other efficient fair algorithms.} The main challenge in \emph{efficient} metric-fair learning is that the metric-fairness constraints are specified in terms of the hypotheses themselves. This means that when the hypotheses are not convex (e.g, the class of logistic predictors), the fairness constraints specify a non-convex set. %While standard convex optimization techniques can still be employed, the theoretical guarantees for fairness and accuracy are not clear.
In the case of logistic regression we overcame these barriers using \emph{improper learning}.
%In particular, we replaced the task of fair-learning the class of logistic predictors with the task of fair-learning a class of linear predictors with a large norm bound (which we do know how to solve efficiently, because there the fairness constraint \emph{is} convex).
However, computational tractability came at the cost of an increase in the sample complexity,
%, and in this sense this an example of the approach of \emph{trading sample complexity for time complexity}. This yielded an efficient fair-learning algorithm for the class of logistic predictors with a $L$-Lipschitz sigmoid function, for a constant $L$. However, the dependence on $L$ is exponential. Since we think of $L$ as the reciprocal of the margin, this implies that for distributions $\D$ for which the data is separated with a \emph{small margin}, our approach is not viable.
and the resulting algorithm was only polynomial so long as the Lipschitz constant $L$ of the sigmoidal transfer function was small. In particular, we only expect good accuracy in cases where data is linearly separable with large (expected) margins.
The hardness result in \cite{shalev2011learning} suggest that a polynomial dependence on $L$ cannot be achieved using this method. A natural question, therefore, is whether other approaches can be used to construct  metric-fair learning algorithms which are efficient and can be accurate even for data separated by smaller margins.

\section{Acknowledgements}
We thank Cynthia Dwork and Omer Reingold for invaluable and illuminating conversations.

%\newpage

\bibliographystyle{alpha}
\bibliography{refs}

\end{document}